\documentclass[11pt]{article}

\usepackage[margin=1in]{geometry}
\usepackage{amsmath,amssymb,amsthm}
\usepackage{mathtools}
\usepackage{graphicx}
\usepackage{enumitem}
\usepackage{natbib}
\usepackage{booktabs}
\usepackage{algorithm}
\usepackage{algpseudocode}
\usepackage{caption}
\usepackage{placeins}
\usepackage{authblk}
\usepackage[colorlinks=true,linkcolor=blue,citecolor=blue,urlcolor=blue]{hyperref}
\usepackage{microtype}


\newtheorem{theorem}{Theorem}[section]

\title{Zero-Inflated Gaussian Distributions Enable Parameter-Space\\
Sparsity in Estimation-of-Distribution Algorithms}

\author[1]{Andreas Faust\thanks{Corresponding author: \href{mailto:andreas.s.faust@gmail.com}{andreas.s.faust@gmail.com}}}
\author[2]{Sven Nitzsche}
\author[3]{Juergen Becker}
\affil[1]{University of Freiburg, Freiburg, Germany}
\affil[2]{FZI Research Center for Information Technology, Karlsruhe, Germany}
\affil[3]{Karlsruhe Institute of Technology, Karlsruhe, Germany}

\date{}

\begin{document}

\maketitle

\begin{abstract}
Estimation-of-distribution algorithms (EDAs) are a powerful class of evolutionary methods for black-box optimization, especially when little is known about the structure of the objective. Whereas classical evolutionary algorithms rely on hand-designed mutation and crossover operators, hard to devise for unknown problem structures, and a source of bias, EDAs sidestep operator design entirely: they fit a probability distribution to the best individuals and sample the next generation from it. EDAs are well established on continuous parameter spaces, but they have not previously been generalized to sparse ones, in which most coefficients of a good solution are exactly zero. Existing sparse black-box optimizers therefore reintroduce exactly what EDAs were designed to avoid: hand-crafted sparsity operators, bi-level schemes alternating between support set and active values, zeroing thresholds, and other baked-in assumptions. We close this gap by proposing multivariate zero-inflated Gaussian (ZIG) distributions as EDA sampling laws. A latent Gaussian model with separate indicator and value dimensions represents sparsity patterns, correlations among active parameters, and the interactions between the two, so sparsity patterns and active values are optimized jointly, hierarchy-free. We show that the latent parameters of this model are identifiable from observed samples, unlike in the missing-data settings where related constructions originate, and introduce practical amortized inversion-based estimators for them. The estimators accurately recover latent correlation structures, and on the Lunar Lander benchmark the resulting ZIG-EDA converges faster and reaches higher final returns than a dense Gaussian EDA, a hand-crafted sparse evolutionary algorithm, and an ad-hoc sparse EDA, while finding controllers with only a small fraction of parameters active.
\end{abstract}

\section{Introduction}
\label{sec:intro}

Many optimization problems have solutions that are naturally sparse: only a small subset of the available parameters needs to be active, while the remaining parameters are exactly zero. Sparse solutions are desirable for several reasons. First, sparsity acts as a form of regularization, reducing overfitting and promoting generalization \citep{tibshirani1996regression,ravikumar2009sparse}. Second, it enables interpretability by highlighting which parameters matter \citep{mitchell1988,george1993}. Third, sparse solutions are computationally efficient, as they reduce model complexity and evaluation costs \citep{donoho_compressed_2006,han2015deepcompression}. Examples include variable selection in regression \citep{tibshirani1996regression}, compression of neural networks \citep{han2015deepcompression}, and portfolio optimization in finance \citep{loris2009sparse}.

This work targets a specific problem class: \emph{fixed-dimensional, sparse, gradient-free black-box optimization}, where no additional assumptions about the problem structure---such as separability or compositionality of the parameters---can be made. Sparsity in optimal solutions often coincides with discontinuities in the objective function, so zeroth-order optimization strategies are needed in the general case \citep{Conn2009,Spall2003,GhadimiLan2013}. For dense continuous spaces, estimation-of-distribution algorithms (EDAs) provide exactly the kind of assumption-free optimizer this problem class calls for. An EDA replaces the variation operators of a population-based evolutionary algorithm with a single statistical step: in each generation, a probability distribution is estimated over the surviving individuals, and the next generation is sampled from that distribution. Instead of relying on hand-designed mutation and crossover operators---which are difficult to devise when the problem structure is unknown, and which bias the search---the statistics of the best solutions found so far drive the search directly. However, EDAs have not previously been generalized to sparse parameter spaces. As we review in Section~\ref{sec:background}, existing approaches to sparse black-box optimization instead reintroduce additional assumptions and design choices: zeroing thresholds, switching conditions for bi-level schemes, hand-crafted sparsity-aware mutation operators, or other hyperparameters that couple the optimizer to the problem. This work originated from the concrete need for an evolutionary optimizer for sparse parameter spaces that, like an EDA in the dense continuous case, requires none of these---to the best of our knowledge, no such method existed.

\paragraph{Contributions.}
We propose the use of zero-inflated Gaussian distributions as sampling laws in EDAs. A zero-inflated Gaussian places a point mass at exactly zero next to a continuous Gaussian component, and therefore provides a natural way to encode sparsity: the point mass represents inactive (``off'') parameters, while the Gaussian component represents active values. In the context of optimization, this corresponds to explicitly searching over sparse candidate solutions rather than enforcing sparsity through a separate constraint or penalty schedule. Concretely, we make the following contributions:
\begin{itemize}[itemsep=0.2em]
    \item We formulate a multivariate zero-inflated Gaussian distribution through a latent Gaussian model with two latent variables per observed dimension---one governing the zero indicator, one the active value---embedded in a full latent correlation model (Section~\ref{sec:model}). The distribution can represent dependencies of three types simultaneously: between continuous values (value--value), between sparsity indicators (mask--mask), and between sparsity and continuous values (value--mask). Sparsity patterns and active parameters are thus optimized jointly rather than in alternating stages. Our proposed sampling law requires no additional hyperparameters or assumptions: it drives the search directly with the statistics of the best solutions found so far, dynamically adapting to the dependencies between sparsity structure and active values of each specific problem.
    \item We show that in the EDA setting the latent parameters of this model are identifiable from observed samples, and we clarify the role of the itemwise conditional independence (ICIN) assumption, which in our setting resolves a parameter redundancy rather than restricting the model class (Section~\ref{sec:identifiability}, with proofs in Appendix~\ref{appendix:identifiability}).
    \item We derive exact attenuation relations between latent and observed correlations and introduce practical amortized inversion-based estimators that recover the latent correlation matrix from data (Section~\ref{sec:estimation}, with derivations in Appendix~\ref{appendix:attenuation}).
    \item We evaluate the approach empirically (Section~\ref{sec:experiments}): the estimators accurately recover latent correlation structures across a broad range of configurations, and on the \texttt{LunarLanderContinuous-v3} benchmark the resulting ZIG-EDA converges faster and attains higher final returns than a dense Gaussian EDA, a hand-crafted sparse evolutionary algorithm, and an ad-hoc sparse EDA, while producing controllers with only about 12 of 90 parameters active.
\end{itemize}

To the best of our knowledge, this is the first connection between the latent Gaussian models studied in the missing-data literature and sparse evolutionary optimization. An open-source implementation of the proposed distribution and its estimators is available at \url{https://github.com/ASFaust/SparseDataModel}.

\paragraph{Organization.}
Section~\ref{sec:background} reviews EDAs, existing approaches to sparse optimization, and the model families our distribution builds on. Sections~\ref{sec:model}--\ref{sec:estimation} develop the method in three steps that mirror what an EDA needs from its sampling law: Section~\ref{sec:model} defines the distribution we sample candidate solutions from; Section~\ref{sec:identifiability} establishes that its parameters can be recovered from samples at all---the prerequisite for refitting the distribution to the elite individuals in every generation; and Section~\ref{sec:estimation} derives the practical estimators that perform this fit. Section~\ref{sec:experiments} evaluates the estimators in isolation and the resulting ZIG-EDA against baselines, and Section~\ref{sec:conclusion} concludes.

\section{Background and Related Work}
\label{sec:background}

\subsection{Estimation-of-distribution algorithms}
\label{sec:eda-background}

Evolutionary algorithms explore a search space by mutating and recombining candidate solutions iteratively, using hand-crafted variation operators. Designing appropriate operators is difficult when little is known about the objective, and poorly chosen operators slow the search down or bias it toward particular kinds of solutions. Estimation-of-distribution algorithms \citep[EDAs;][]{LarranagaLozano2002} take a different approach: instead of explicitly designing variation operators, they maintain an explicit probability distribution over the search space. In each generation, an EDA (i) samples a population of candidate solutions from its current distribution, (ii) evaluates them on the objective, (iii) selects the most successful individuals, and (iv) refits the distribution to the selected individuals. The fitted distribution thus accumulates information about the structure of high-quality solutions, and the search progresses by alternating between sampling and refitting.

The choice of distribution family determines what structure an EDA can exploit. Early EDAs used independent marginals; later variants capture dependencies between variables, for example through Bayesian networks over discrete variables \citep{PelikanGoldbergCantuPaz1999BOA} or multivariate Gaussians over continuous ones. Copula-based EDAs \citep{Soto2012ModelingWithCopulas} decouple the marginal distributions from the dependence structure and can model arbitrary continuous marginals. However, such methods treat zero as just another value on the real line rather than as an explicit indicator of sparsity. As a result, they lack a mechanism to distinguish structural zeros from ordinary small values, and the probability of sampling an exact zero is null under any continuous model.

\subsection{Algorithms exploring sparse solutions}
\label{sec:sparse-algos}

Several strategies exist for finding sparse solutions with evolutionary methods. One line of work designs mutation and crossover operators that explicitly enforce sparsity in the offspring \citep{Wang2025SparseEA_AGDS}. Another family adopts a hierarchical or bi-level design, alternating between optimizing a sparse support set (e.g., a binary mask) and optimizing the continuous parameters within that support, effectively separating the discrete sparsity structure from the continuous search \citep{zhang2022bip,li2023differentiable_transportation_pruning,guan2020dais}. All of these approaches require additional assumptions and design choices---zeroing thresholds, switching conditions between levels, or hand-crafted operators---whose settings couple the algorithm to the problem instance. In contrast, the sampling law proposed in this work drives the search directly with the statistics of the selected individuals and adapts the coupling between sparsity pattern and active values automatically.

\subsection{Zero-inflated and spike-and-slab models}
\label{sec:zig-background}

Distributions that mix a point mass at zero with a continuous component appear in several literatures under different names. In Bayesian variable selection, \emph{spike-and-slab priors} place such a mixture on regression coefficients to infer which covariates enter a model \citep{mitchell1988,george1993}. In applied statistics, \emph{zero-inflated} models describe observed data with excess zeros, treating the mixture as a likelihood or data model rather than as a prior. Although both share the same marginal functional form, the two framings differ fundamentally: a spike-and-slab prior encodes beliefs about parameters and is updated through a likelihood, whereas a zero-inflated model directly describes observable quantities.

Our use of the mixture falls squarely into the second category. We never perform Bayesian updating; instead, we fit the distribution to the genomes of selected individuals and sample new candidates from it. The distribution is a \emph{sampling law over candidate solutions}---a data model---and we therefore adopt the name \emph{zero-inflated Gaussian} (ZIG) throughout this paper.

A second relevant line of work comes from missing-data statistics. \citet{rubin1976} introduced the general framework for inference under missingness mechanisms. Recent approaches such as \citet{feldman2024} combine mixed marginals with Gaussian copula models in order to capture dependencies between observed values and missingness patterns, and \citet{ICIN} introduced the itemwise conditionally independent nonresponse (ICIN) assumption to obtain identifiability in nonignorable missing-data models. Our model construction---a latent Gaussian variable thresholded to decide whether a value is zeroed---is mathematically close to these probit-style missingness models. The crucial difference lies in the interpretation: in missing-data settings the masked values are real but unobserved, and the analyst's goal is to reason about them; in our setting a zero is a structural feature of the candidate solution itself, and no ``true value behind the zero'' exists. As we show in Section~\ref{sec:identifiability}, this difference makes the latent parameters directly identifiable in our setting, with ICIN serving only to fix a parameter redundancy.

\begin{figure}[t]
    \centering
    \includegraphics[width=0.62\linewidth]{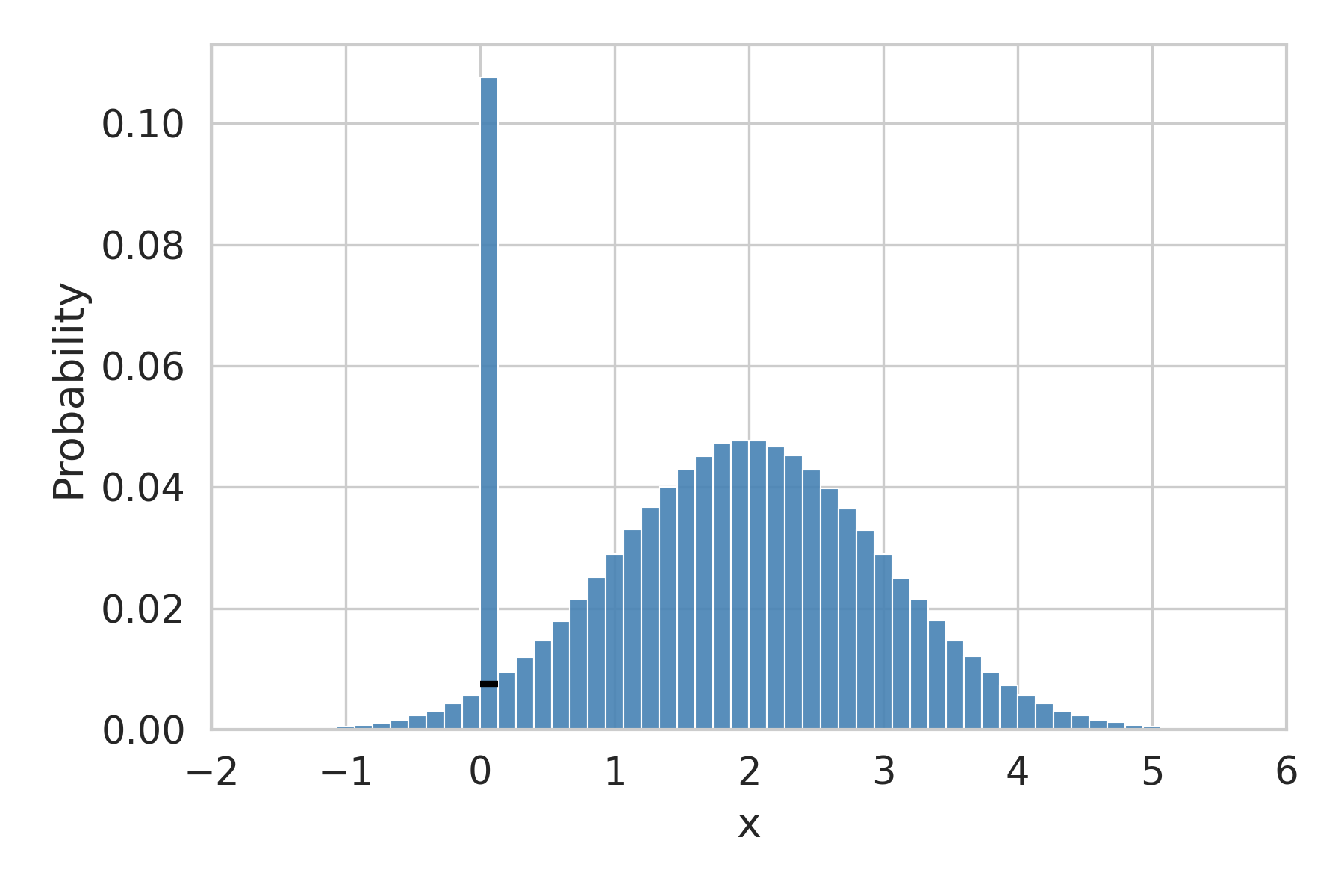}
\caption{Histogram of a univariate zero-inflated Gaussian with $10\%$ zero-inflation and a Gaussian component with $\mu = 2, \sigma = 1$.
The bar at zero is visually split to distinguish the exact zeros contributed by the point mass from the near-zero values originating from the Gaussian component.}
    \label{fig:zig_hist_01}
\end{figure}

Figure~\ref{fig:zig_hist_01} illustrates the defining characteristic of a zero-inflated Gaussian: a fraction of the probability mass is concentrated at exactly zero, while the rest is spread over a Gaussian.

\section{The Zero-Inflated Gaussian Model}
\label{sec:model}

We represent each dimension of a multivariate zero-inflated Gaussian by explicitly separating the zero component and the Gaussian component. The point mass is fixed at zero, while the active values are modeled as Gaussian. This ensures a clear separation of the two components: since a Gaussian assigns zero probability to any single value, observations at exactly zero can only come from the point mass, while all nonzero values belong to the Gaussian component.

Formally, for $d$ observed dimensions we introduce $2d$ latent Gaussian variables:
a \emph{value-latent} $V_i$ for each dimension $i \in \{1,\dots,d\}$, and a \emph{mask-latent} $M_i$ that determines whether the observed entry is set to zero or drawn from the Gaussian component. We collect these into
\[
Z = (V_1, \dots, V_d, \, M_1, \dots, M_d)^\top \sim \mathcal{N}(0, \Sigma),
\]
with $\Sigma \in \mathbb{R}^{2d \times 2d}$ a correlation matrix capturing dependencies among all latents.

Each observation $X_i$ is then generated as
\[
X_i =
\begin{cases}
0, & \text{if } M_i < \tau_i, \\
\mu_i + \sigma_i V_i, & \text{otherwise},
\end{cases}
\]
where $\tau_i \in \mathbb{R}$ is a threshold controlling the probability of a zero, and $(\mu_i, \sigma_i)$ are the parameters of the Gaussian component. We call $p_i = \Pr(X_i \neq 0)$ the \emph{activation probability} of dimension $i$. Finally, the observed vector is
\[
X = (X_1, \dots, X_d)^\top.
\]

Because all $2d$ latents share one correlation matrix, the model represents dependencies of three types simultaneously: between continuous values (value--value), between sparsity indicators (mask--mask), and between sparsity and continuous values (value--mask). In this way, sparsity patterns and active parameters can be optimized jointly rather than in alternating stages, while preserving the ability to capture interactions both across and within modes.

\section{Identifiability and the ICIN Assumption}
\label{sec:identifiability}

Inside an EDA loop, the distribution of Section~\ref{sec:model} is refitted to the elite individuals in every generation, so everything hinges on being able to estimate its parameters from samples. Before we can do so, we have to recognize that several different sets of latent parameters of the described model can create the same observed distribution. The ambiguity arises from correlations between the value-latent $V_i$ and the mask-latent $M_i$ of the same dimension. These correlations shift the observed mean and variance of the Gaussian component away from the latent mean and variance of $V_i$, by biasing which parts of $V_i$ are masked how often.

For example, a large positive correlation causes low value-latents to often coincide with low mask-latents, masking the lower part of the value-latent more often than the upper part. Therefore, any observed marginal distribution can be explained by multiple different combinations of latent $\operatorname{corr}(V_i,M_i)$ and $(\mu_i,\sigma_i)$. Figure~\ref{fig:pdf_correlation} illustrates this effect.

\begin{figure}[t]
    \centering
    \includegraphics[width=0.62\linewidth]{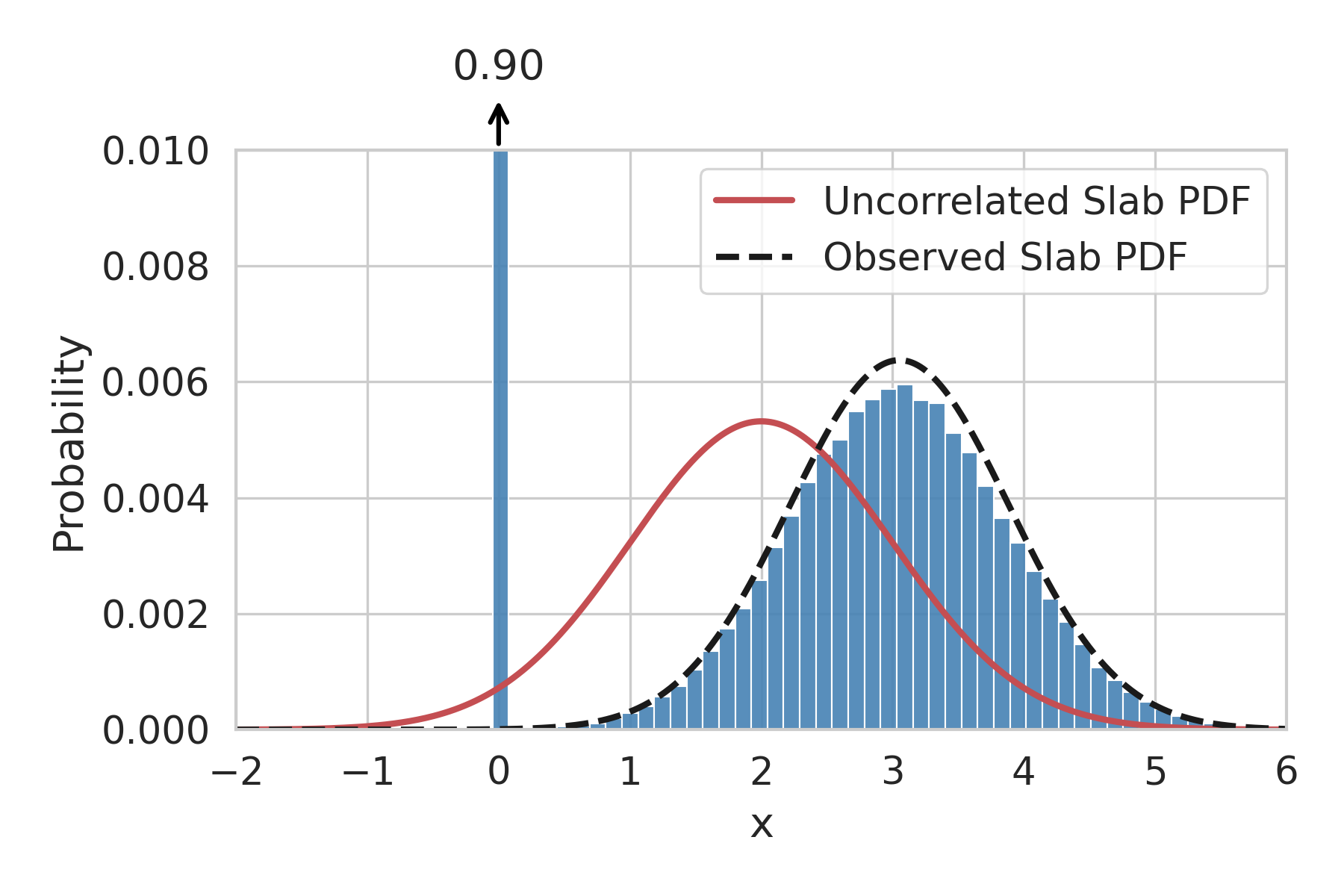}
    \caption{Effect of correlation between $V_i$ and $M_i$.
    Parameters: activation probability $p = 0.1$, $\mu = 2$, $\sigma = 1$, $\operatorname{corr}(V_i,M_i) = 0.6$.
    The red curve shows what the density of the Gaussian component would be if $\operatorname{corr}(V_i,M_i)=0$, in which case $\mu$ and $\sigma$ are directly recoverable.
    The dashed curve shows the observed density under positive correlation: the estimated $\hat{\mu}$ and $\hat{\sigma}$ are shifted relative to the latent values.}
    \label{fig:pdf_correlation}
\end{figure}

To resolve this non-identifiability, we impose the \emph{itemwise conditional independence} (ICIN) assumption \citep{ICIN}:
\[
\operatorname{corr}(V_i, M_i) = 0 \quad \text{for all } i \in \{1,\dots,d\}.
\]
Correlations between latents of \emph{different} dimensions, including value--mask correlations $\operatorname{corr}(V_i, M_j)$ for $i \neq j$, remain unrestricted.

This assumption is essential for identifiability, but its implications differ across domains. In missing-data settings, analysts interpret $V_i$ as the unobserved ``true'' values before censoring, and try to infer the masked values for imputation purposes. But without auxiliary information, distinguishing between correlations of $M_i$ and $V_i$ and shifts in $\mu_i$ and $\sigma_i$ is impossible. \citet{feldman2024} address this problem by incorporating auxiliary quantile information, which enables them to estimate correlations between value- and mask-latents that would otherwise remain hidden. In optimization contexts, however, we do not attribute meaning to a pre-masked latent distribution: the zero-inflated Gaussian law itself is the object of interest. Therefore, ICIN serves only to resolve non-identifiability and carries no substantive consequences.

Another form of ambiguity can arise in degenerate cases, if some observed marginal is always zero or always Gaussian. Such dimensions can be modeled with simpler marginal models (either modeling the dimension as constantly $0$, or as purely Gaussian by removing $M_i$).

Formal proofs demonstrating both the necessity and sufficiency of ICIN for model identifiability are provided in Appendix~\ref{appendix:identifiability}.

\section{Estimation of Model Parameters}
\label{sec:estimation}

Given a dataset $X \in \mathbb{R}^{n \times d}$ of observed zero-inflated Gaussian values,
we aim to estimate the parameters $(\mu_i, \sigma_i, \tau_i)$ of each marginal distribution as well as the latent correlation matrix $\Sigma \in [-1,1]^{2 d \times 2 d}$ of the latent Gaussian variables.

\subsection{Estimating the marginal parameters}

Let $p_i$ denote the empirical activation probability of dimension $i$, given by the fraction of nonzero values in $X_i$.
This activation probability determines the threshold $\tau_i$ on the mask-latent $M_i$ via
\[
\tau_i = \Phi^{-1}(1-p_i),
\]
where $\Phi^{-1}$ is the inverse cumulative distribution function of the standard normal.

Under ICIN, we can directly estimate the dimension-wise parameters of the Gaussian component from the observed data $X_i$ as follows:
\begin{align*}
\hat{\mu}_i &= \operatorname{mean}\{X_i \;|\; X_i \neq 0\}, \\
\hat{\sigma}_i &= \operatorname{std}\{X_i \;|\; X_i \neq 0\}.
\end{align*}

\subsection{Estimating the correlation matrix}
\label{sec:corr-estimation}

To estimate the correlation matrix, we map each observed $X_i$ to a standardized value $\tilde{X}_i$ and a binary indicator $I_i$:
\[
\tilde{X}_i =
\begin{cases}
0, & X_i = 0, \\
\dfrac{X_i - \hat{\mu}_i}{\hat{\sigma}_i}, & X_i \neq 0,
\end{cases}
\qquad
I_i =
\begin{cases}
0, & X_i = 0, \\
1, & X_i \neq 0.
\end{cases}
\]

We then compute an initial, \emph{na\"ive} Pearson correlation matrix $P$ from the extended observed data
\[
Y = (\tilde{X}_1,\dots,\tilde{X}_d,\, I_1,\dots,I_d).
\]

$P$ and $\Sigma \in \mathbb{R}^{2d \times 2d}$ can both be partitioned into four blocks:
\[
\begin{aligned}
P &=
\begin{bmatrix}
P_{VV} & P_{VM} \\
P_{VM}^\top & P_{MM}
\end{bmatrix},
&
\Sigma &=
\begin{bmatrix}
\Sigma_{VV} & \Sigma_{VM} \\
\Sigma_{VM}^\top & \Sigma_{MM}
\end{bmatrix},
\end{aligned}
\]
corresponding to correlations among continuous values ($VV$), among masks ($MM$), and across the two types ($VM$).

\subsubsection{Attenuation and inverse mappings}
\label{sec:attenuation}

Due to the thresholding and masking operations in the generation of $X$,
the empirical correlation blocks $P_{\cdot\cdot}$ are systematically
distorted with respect to their latent counterparts $\Sigma_{\cdot\cdot}$.

To reverse these distortions, we derive integral expressions for the
expected empirical correlations
\[
\mathbb{E}\!\left[P_{\cdot\cdot}(i,j)\right]
= f_{\cdot\cdot}\!\left(p_i,\,p_j,\,\Sigma_{\cdot\cdot}(i,j),\,\Sigma_{\text{aux}}\right),
\]
where \(p_i=\Pr(X_i\ne0)\) encodes the marginal activation probability,
and \(\Sigma_{\text{aux}}\) collects auxiliary latent correlations
relevant for the block type (e.g.\ $\Sigma_{MM}$ in value--mask terms).
Each function \(f_{\cdot\cdot}\) is strictly monotone in the latent correlation
of interest, ensuring the existence of a unique inverse mapping.
All integral expressions and monotonicity proofs are provided in Appendix~\ref{appendix:attenuation}.

By recursively substituting previously identified quantities
(e.g.\ replacing $\Sigma_{MM}$ by its observable form),
we obtain minimal inverse relations depending only on observables.
The resulting identification problems are:
\[
\begin{aligned}
\text{Mask--mask:}\quad
&(p_i,\,p_j,\,P_{MM}(i,j))
\;\mapsto\; \Sigma_{MM}(i,j),
\\[0.4em]
\text{Value--mask:}\quad
&(p_i,\,p_j,\,P_{VM}(i,j),\,P_{MM}(i,j))
\;\mapsto\; \Sigma_{VM}(i,j),
\\[0.4em]
\text{Value--value:}\quad
&(p_i,\,p_j,\,P_{VV}(i,j),\,P_{VM}(i,j),\,P_{VM}(j,i),\,P_{MM}(i,j))
\;\mapsto\; \Sigma_{VV}(i,j).
\end{aligned}
\]
Thus, each latent correlation is uniquely determined by its
empirical counterpart and a minimal set of auxiliary observables.

\subsubsection{Amortized inversion}
\label{sec:amortized}

To obtain a practical estimator of $\Sigma$ from $P$,
we employ amortized inference over synthetic training data.
A large dataset is generated with known latent correlations $\Sigma$,
from which the corresponding empirical correlations $P$ are computed.
Function approximators are then trained to learn the previously defined mappings for each block type separately.

We use compact multilayer perceptrons (MLPs) trained via gradient descent to minimize mean-squared reconstruction error.
Because all involved quantities (\(P,\Sigma,p_i,p_j\)) are bounded,
the synthetic dataset fully covers the observable manifold,
and the learned mappings generalize across arbitrary dimensionalities
when applied pairwise to $(i,j)$-entries of new zero-inflated Gaussian data.
Training details can be found in Appendix~\ref{appendix:training}.

\subsubsection{Projection to the correlation manifold}

The estimated matrix $\hat{\Sigma}$ is not guaranteed to be positive semi-definite, due both to estimation errors in $P$ and approximation errors in the learned mappings. We therefore project $\hat{\Sigma}$ to the nearest valid correlation matrix in three steps:
\begin{enumerate}[label=(\roman*)]
    \item \textbf{Symmetrization:} set $\hat{\Sigma}_s = (\hat{\Sigma} + \hat{\Sigma}^\top)/2$,
    \item \textbf{Eigenvalue correction:} compute the eigendecomposition $\hat{\Sigma}_s = V \Lambda V^\top$ and replace negative entries of $\Lambda$ by a small constant, yielding $\hat{\Sigma}_+ = V \Lambda_+ V^\top$,
    \item \textbf{Diagonal normalization:} define $D = \operatorname{diag}(1/\sqrt{(\hat{\Sigma}_+)_{11}}, \dots, 1/\sqrt{(\hat{\Sigma}_+)_{dd}})$ and set $\hat{\Sigma}^\star = D \hat{\Sigma}_+ D$, which enforces $\operatorname{diag}(\hat{\Sigma}^\star) = (1,\dots,1)$.
\end{enumerate}
The result $\hat{\Sigma}^\star$ is symmetric positive semi-definite with unit diagonal entries, i.e.\ a valid correlation matrix.

\subsection{Computational complexity}

Runtime scales with sample size $n$ and dimensionality $d$.
Computing marginal statistics is $\mathcal{O}(nd)$.
The dominant cost is forming the Pearson correlation matrix
$P \in \mathbb{R}^{2d \times 2d}$, requiring $\mathcal{O}(n d^2)$ operations.
Mapping $P \mapsto \Sigma$ adds $\mathcal{O}(d^2)$ for pairwise MLP evaluations,
and the projection step costs $\mathcal{O}(d^3)$ due to eigendecomposition.
Overall complexity:
\[
\mathcal{O}(n d^2 + d^3),
\]
typically dominated by $\mathcal{O}(n d^2)$, since $n \gg d$ for stable correlation estimates.

\section{Experimental Results}
\label{sec:experiments}

\subsection{Recovering latent parameters}
\label{sec:exp-recovery}

We evaluate the accuracy of the proposed estimation procedure in recovering the latent parameters of the model.
To this end, we instantiate $100$ latent Gaussian models with randomly chosen $\Sigma$, $\mu$, $\sigma$, and $p$.
The observed dimensionality is set to $32$, corresponding to a latent dimensionality of $64$.
From each model, we sample $100{,}000$ observations.
This setup yields
\[
\Bigl(\frac{64\,(64-1)}{2} - 32\Bigr) \times 100 \;=\; 198{,}400
\]
to-be-recovered correlations.

Figure~\ref{fig:hexbin_corr} compares the recovered correlations $\hat{\rho}$ with the ground-truth $\rho$.
A quantitative summary of the recovery performance is given in Table~\ref{tab:corr_metrics}.
The concordance correlation coefficient (CCC) is more informative than the Pearson correlation $r$ in this setting, as CCC measures agreement with the identity function rather than mere linear association.

\begin{table}[t]
\caption{Summary statistics over all recovered correlations.}
\label{tab:corr_metrics}
\begin{center}
\begin{tabular}{cccc}
\toprule
\textbf{Mean($|\rho|$)} & \textbf{Mean($|\hat{\rho}-\rho|$)} & \textbf{MSE} & \textbf{CCC} \\
\midrule
0.323 & 0.0272 & $1.58\times 10^{-3}$ & 0.994 \\
\bottomrule
\end{tabular}
\end{center}
\end{table}

\begin{figure}[t]
    \centering
    \includegraphics[width=0.6\linewidth]{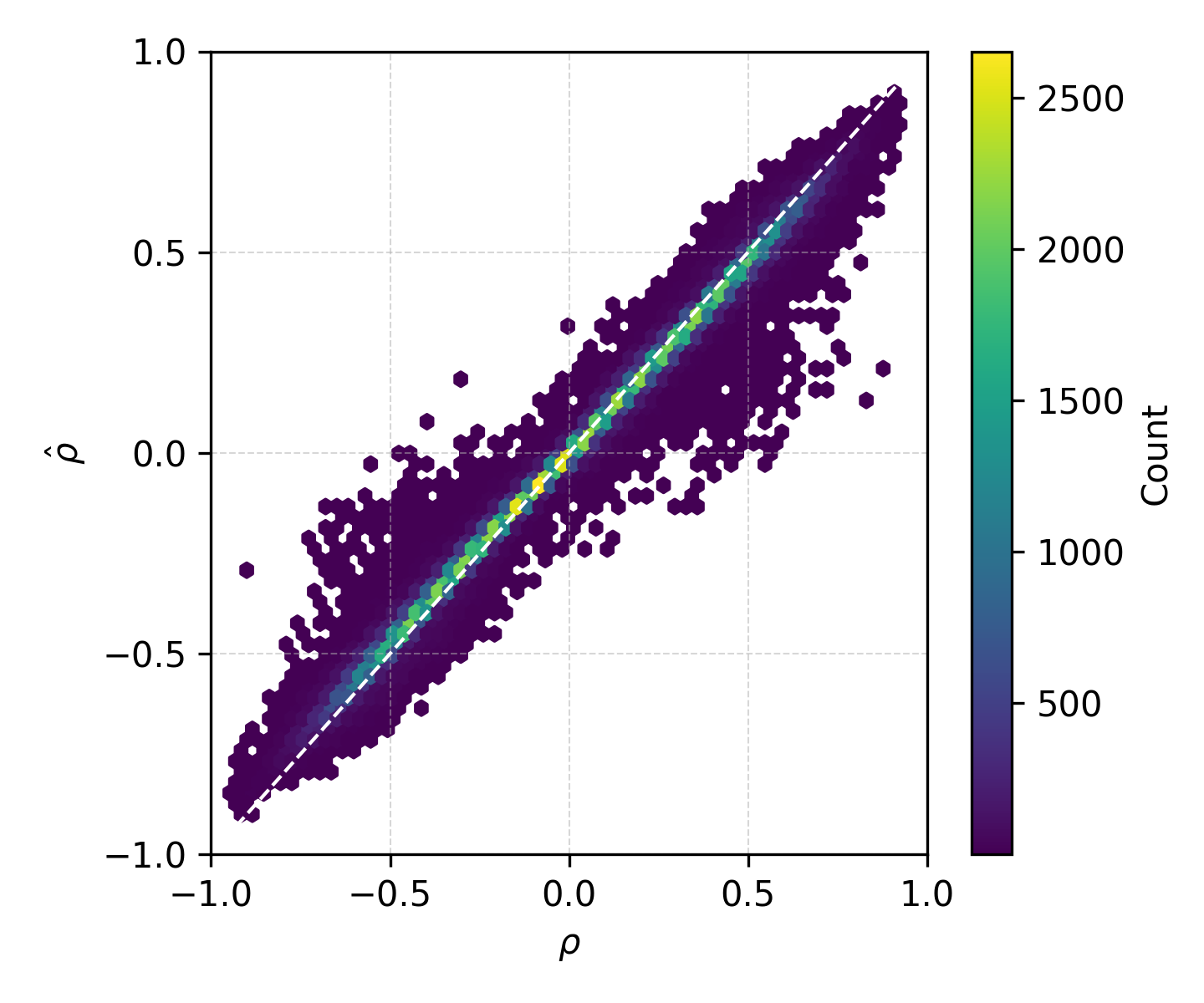}
    \caption{Recovered vs.\ true off-diagonal correlations. The dashed line marks the identity.}
    \label{fig:hexbin_corr}
\end{figure}

The quantitative metrics indicate a close match between true and recovered correlations:
both the mean absolute and mean squared errors are small, while a CCC near unity confirms strong overall agreement.
These results suggest that the estimation procedure accurately reconstructs the latent dependence structure across a broad range of correlation magnitudes.

A block-specific analysis reveals that the residual underestimation is most pronounced in the $\Sigma_{VV}$ block.
Figure~\ref{fig:gg_corr_err} shows absolute correlation errors as a function of $\min(p_i, p_j)$, where $p_i$ denotes the marginal activation probability.
Errors increase substantially when $\min(p_i, p_j)$ is small, meaning that at least one dimension is rarely active.
This pattern reflects the reduced effective sample size in such regions, which limits the precision of correlation recovery.

\begin{figure}[t]
    \centering
    \includegraphics[width=0.6\linewidth]{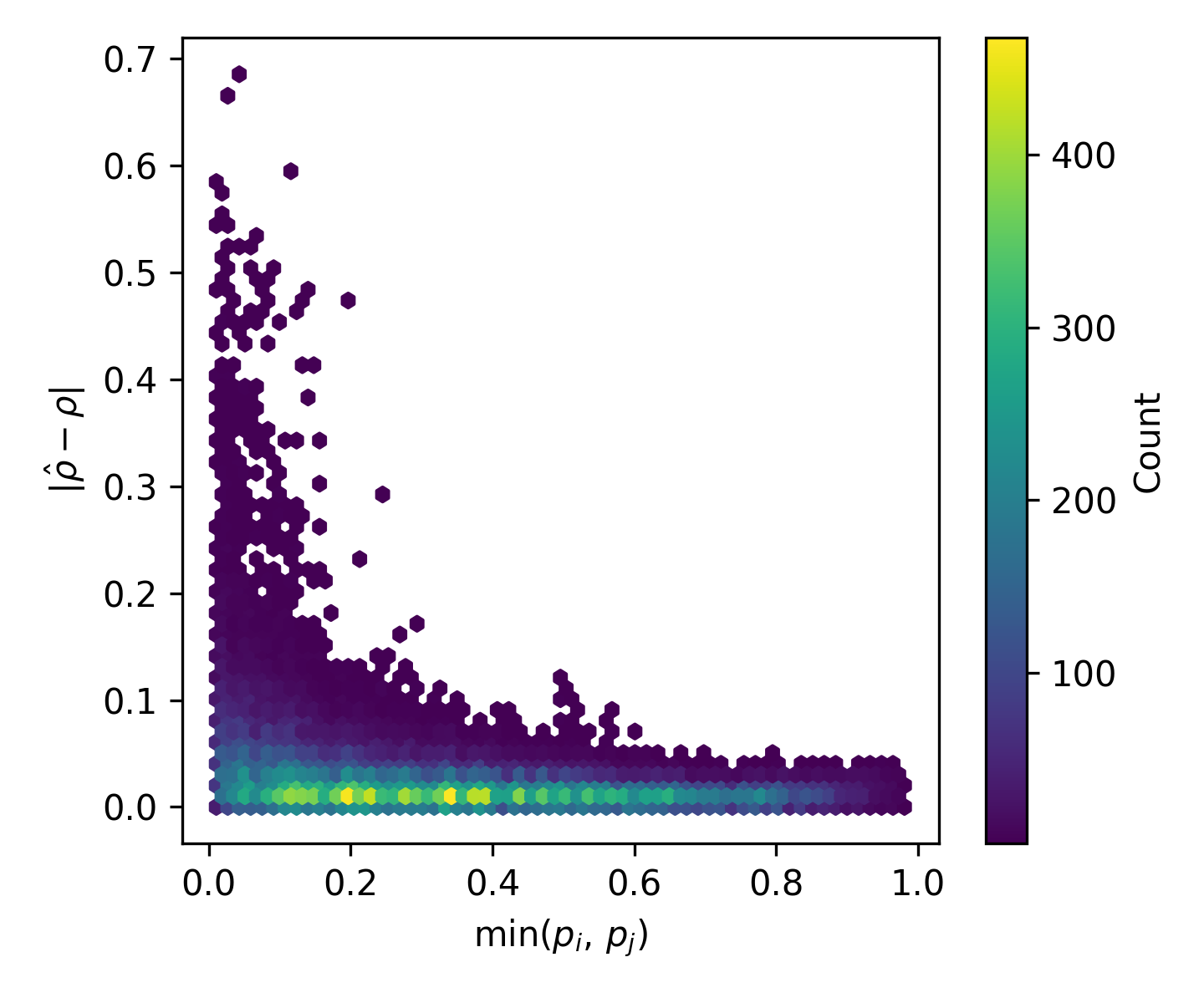}
    \caption{Absolute correlation errors in value--value blocks as a function of $\min(p_i, p_j)$. Larger errors occur when at least one dimension is rarely active.}
    \label{fig:gg_corr_err}
\end{figure}

\FloatBarrier

\subsection{Application to Lunar Lander}
\label{sec:exp-lander}

The \texttt{LunarLanderContinuous-v3} environment from Gymnasium is a widely used reinforcement learning benchmark. The task is to control a simulated spacecraft and guide it to a soft landing between two flags on the lunar surface. The state space consists of eight continuous variables, including position, velocity, orientation, and leg contact indicators. The action space is two-dimensional and continuous, representing throttle values for the main engine and the side engines in the range $[-1,1]$. Most commonly, this benchmark is solved using continuous reinforcement learning algorithms such as policy gradient or actor--critic methods. Although the environment return is not explicitly bounded, the physics and reward-shaping terms implicitly constrain the achievable episode return: in practice, scores around $300$ correspond to near-perfect trajectories.

\subsubsection{Controller}

We parameterize candidate policies as sparse quadratic functions of the observed state. For an eight-dimensional state vector $s = (s_1,\dots,s_8)$, the two action dimensions are computed as
\[
a_k = b_k
  + \sum_{i=1}^8 w_{ki}\, s_i
  + \sum_{1 \le i \le j \le 8} u_{kij}\, s_i s_j,
\qquad
k \in \{\text{main}, \text{side}\},
\]
where $b_k$ is a bias term, $w_{ki}$ are linear coefficients, and $u_{kij}$ are quadratic coefficients. The full controller has $90$ free parameters. Sparsity arises when many of these coefficients are set to zero, leaving only a small number of active terms to determine the policy.

\subsubsection{Algorithm}
\label{sec:algorithm}

We optimize the controller using a simple estimation-of-distribution algorithm
with elite replay and a persistent fitness map.
In each generation, the algorithm refits a zero-inflated Gaussian distribution
to the best previously observed individuals and re-evaluates them together with newly
sampled candidates. This prevents lucky high-fitness outliers from dominating due to stochasticity
and ensures consistent ranking across generations. Algorithm~\ref{alg:eda-general} states the general procedure; Algorithm~\ref{alg:eda-lander} gives the instantiation used for the controller-recovery experiment in Section~\ref{sec:recovered-controller}.

\begin{algorithm}[t]
\caption{Estimation-of-Distribution with Elite Replay and Fitness Map}
\label{alg:eda-general}
\begin{algorithmic}[1]
\State \textbf{Input:} initial population $\mathcal{P}$ with size $N$, elite size $K$, generations $T$
\State Initialize fitness map $\mathcal{F} \gets \{\}$
\For{$t = 1$ to $T$}
  \For{each $\theta \in \mathcal{P}$}
    \State Evaluate fitness $F(\theta)$
    \State $\mathcal{F}[\theta] \gets F(\theta)$
  \EndFor
  \If{$t = T$}
    \State \Return $\arg\max_{\theta \in \mathcal{F}} \mathcal{F}[\theta]$
  \EndIf
  \State $\mathcal{E} \gets \text{select top } K \text{ from } \mathcal{F}$
  \State Fit zero-inflated Gaussian model $\mathcal{D}_t$ on $\mathcal{E}$
  \State $\mathcal{C} \gets \text{sample } (N{-}K) \text{ individuals } \sim \mathcal{D}_t$
  \State $\mathcal{P} \gets \mathcal{E} \cup \mathcal{C}$
\EndFor
\end{algorithmic}
\end{algorithm}

\begin{algorithm}[t]
\caption{Instantiation for Lunar Lander Quadratic Controller}
\label{alg:eda-lander}
\begin{algorithmic}[1]
\State \textbf{Controller:} quadratic policy with $90$ parameters
(bias, linear, and quadratic terms)
\State \textbf{Hyperparameters:} population size $N{=}1000$, elite size $K{=}100$, generations $T{=}100$
\State \textbf{Initial population:} $\mathcal{P}_0$ sampled from a zero-inflated Gaussian with $\mu_0{=}0$, $\sigma_0{=}2$, $p_0{=}0.5$, and $\Sigma_0{=}I$
\State \textbf{Evaluation:} average return $R_{\text{avg}}(\theta)$ over $S{=}8$ random environment initializations
\State \textbf{Fitness:} $F(\theta) = R_{\text{avg}}(\theta) - 2 \cdot (\text{nonzero param count})$
\State \textbf{Procedure:} run Algorithm~\ref{alg:eda-general} with inputs
$\mathcal{P}_0$, $N$, $K$, $T$
\State \textbf{Output:} final controller $\theta^\star = \arg\max_{\theta \in \mathcal{F}} \mathcal{F}[\theta]$
\end{algorithmic}
\end{algorithm}

\subsubsection{Baselines}
\label{sec:baselines}

To contextualize the proposed method, we compare it against three baselines on the same task:
\begin{enumerate}[itemsep=0.2em]
    \item \textbf{Dense EDA:} a classical dense Gaussian EDA that fits a multivariate Gaussian to the elites. It cannot represent exact zeros and therefore optimizes the raw environment return without a sparsity penalty.
    \item \textbf{Sparse EA:} a hand-crafted evolutionary algorithm with manually designed mutation and crossover operators that explicitly zero out and re-activate parameters. The operators are specified in Appendix~\ref{appendix:sparse-ea}.
    \item \textbf{Ad-hoc sparse EDA:} a classical dense Gaussian EDA in which sparsity is encoded by doubling the genome size. Each polynomial coefficient is represented by two normally distributed variables, a weight $w_i$ and a threshold variable $t_i$, and the effective coefficient is obtained as
    \[
    x_i = w_i \cdot \mathbf{1}\{t_i \ge 0\}.
    \]
    This baseline is the closest comparison to our method: it can represent exact zeros and joint statistics over the doubled genome. The core difference is that its correlation matrix is estimated over the dense parametrization, which includes the values $w_i$ of inactive parameters (where $t_i < 0$) and noisy threshold values $t_i$. This muddies the correlation structure, introducing estimation noise and underestimating the relevant correlations.
\end{enumerate}

All methods use the same selection mechanism (top-$k$ over all previously evaluated individuals), number of generations ($100$), number of survivors ($100$), and population size ($500$). The only differences are the sampling laws for the new population and, for the ad-hoc sparse EDA, the genome representation. The three sparse methods optimize the penalized fitness with a penalty of $2$ per active parameter; reported scores are always the unpenalized environment returns. Each optimizer was run $10$ times with different random seeds.

\subsubsection{Comparison results}
\label{sec:comparison-results}

Figure~\ref{fig:score_progression} shows the progression of the environment return of the best individual found so far, and Figure~\ref{fig:sparsity_progression} the corresponding number of active parameters for the three sparsity-aware methods. Figures~\ref{fig:final_scores}--\ref{fig:sparsity_bars} summarize the final scores, the number of generations needed to reach fixed score thresholds, and the sparsity levels at fixed generations. The full numerical results are tabulated in Appendix~\ref{appendix:tables}.

\begin{figure}[t]
    \centering
    \includegraphics[width=0.78\linewidth]{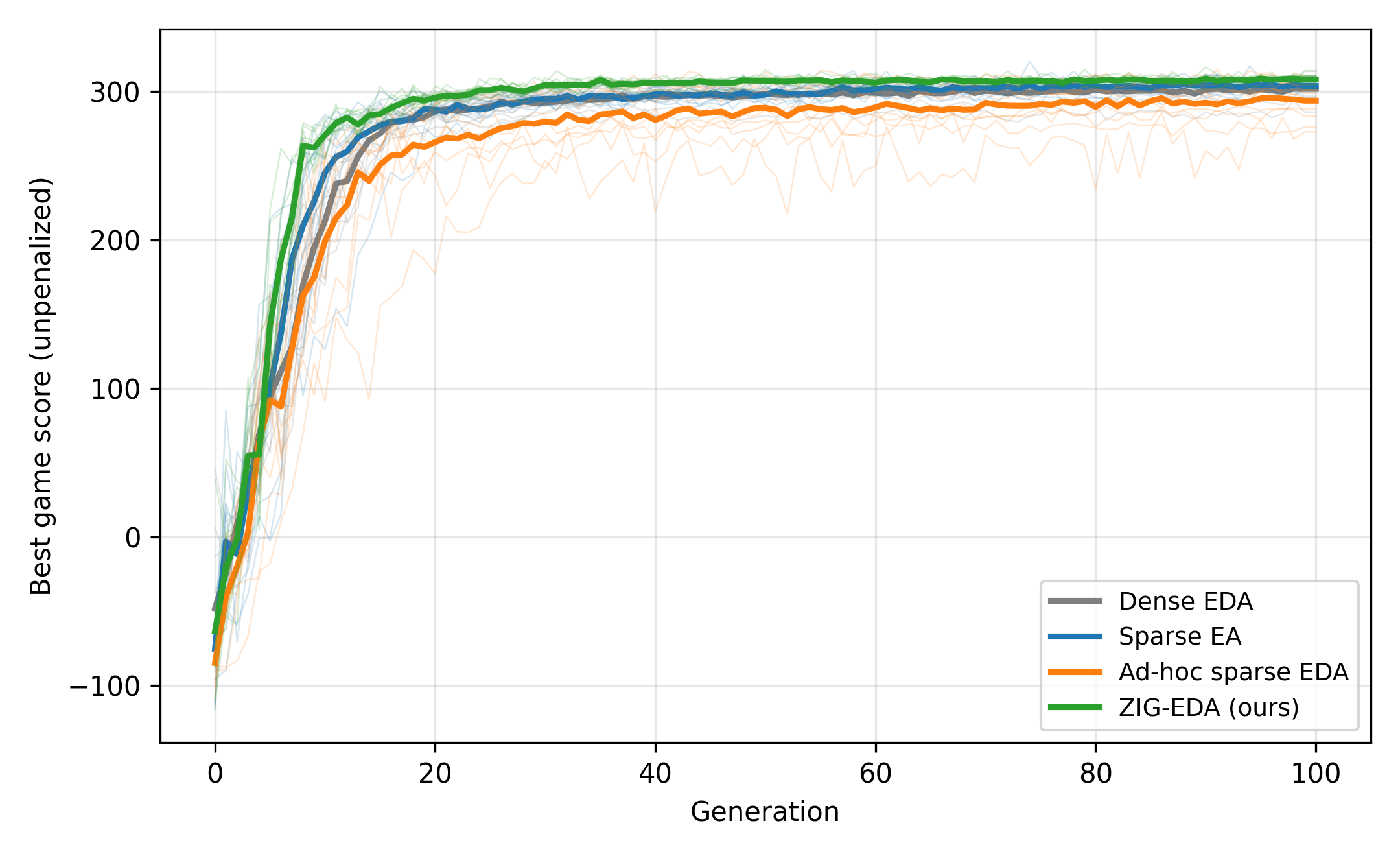}
    \caption{Environment return (unpenalized) of the best individual found so far, over $100$ generations. Thick lines show the mean over $10$ runs; thin lines show individual runs.}
    \label{fig:score_progression}
\end{figure}

\begin{figure}[t]
    \centering
    \includegraphics[width=0.78\linewidth]{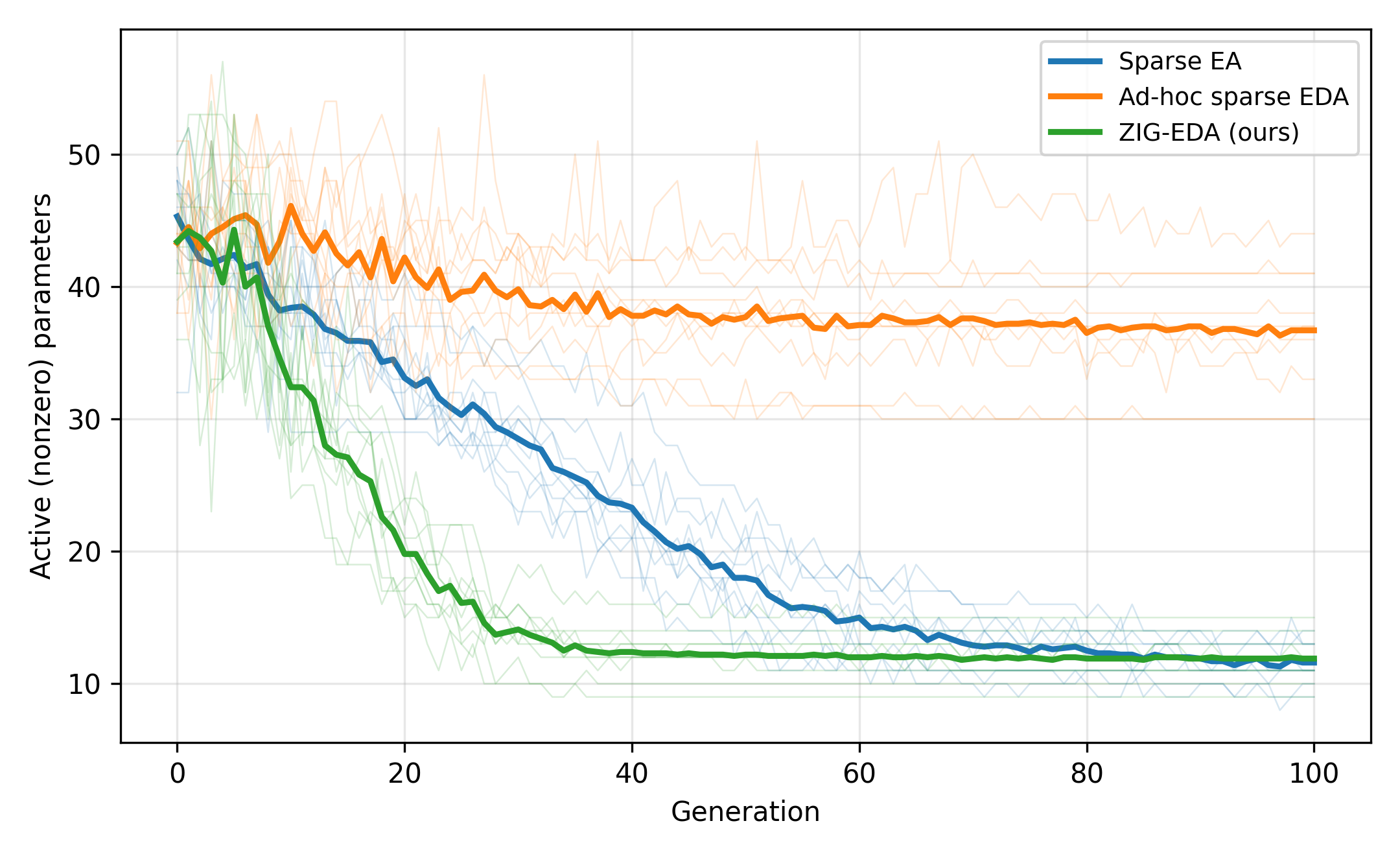}
    \caption{Number of active (nonzero) parameters of the best individual found so far, for the three sparsity-aware methods. The dense EDA always uses all $90$ parameters and is omitted. Thick lines show the mean over $10$ runs; thin lines show individual runs.}
    \label{fig:sparsity_progression}
\end{figure}

\begin{figure}[!htb]
    \centering
    \includegraphics[width=0.62\linewidth]{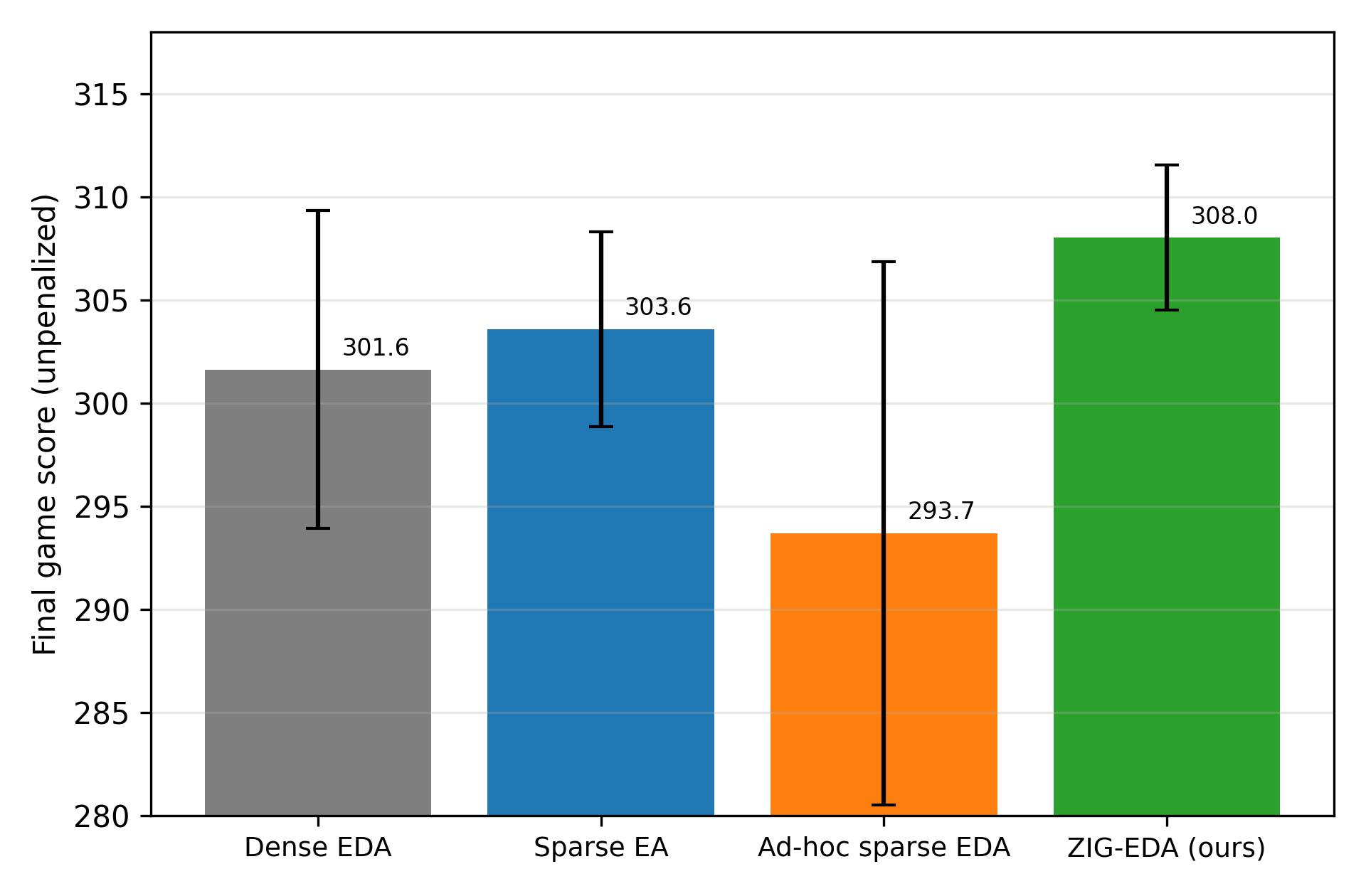}
    \caption{Final unpenalized scores after $100$ generations (mean over $10$ runs; error bars show one standard deviation).}
    \label{fig:final_scores}
\end{figure}

\begin{figure}[t]
    \centering
    \begin{minipage}[t]{0.48\linewidth}
        \centering
        \includegraphics[width=\linewidth]{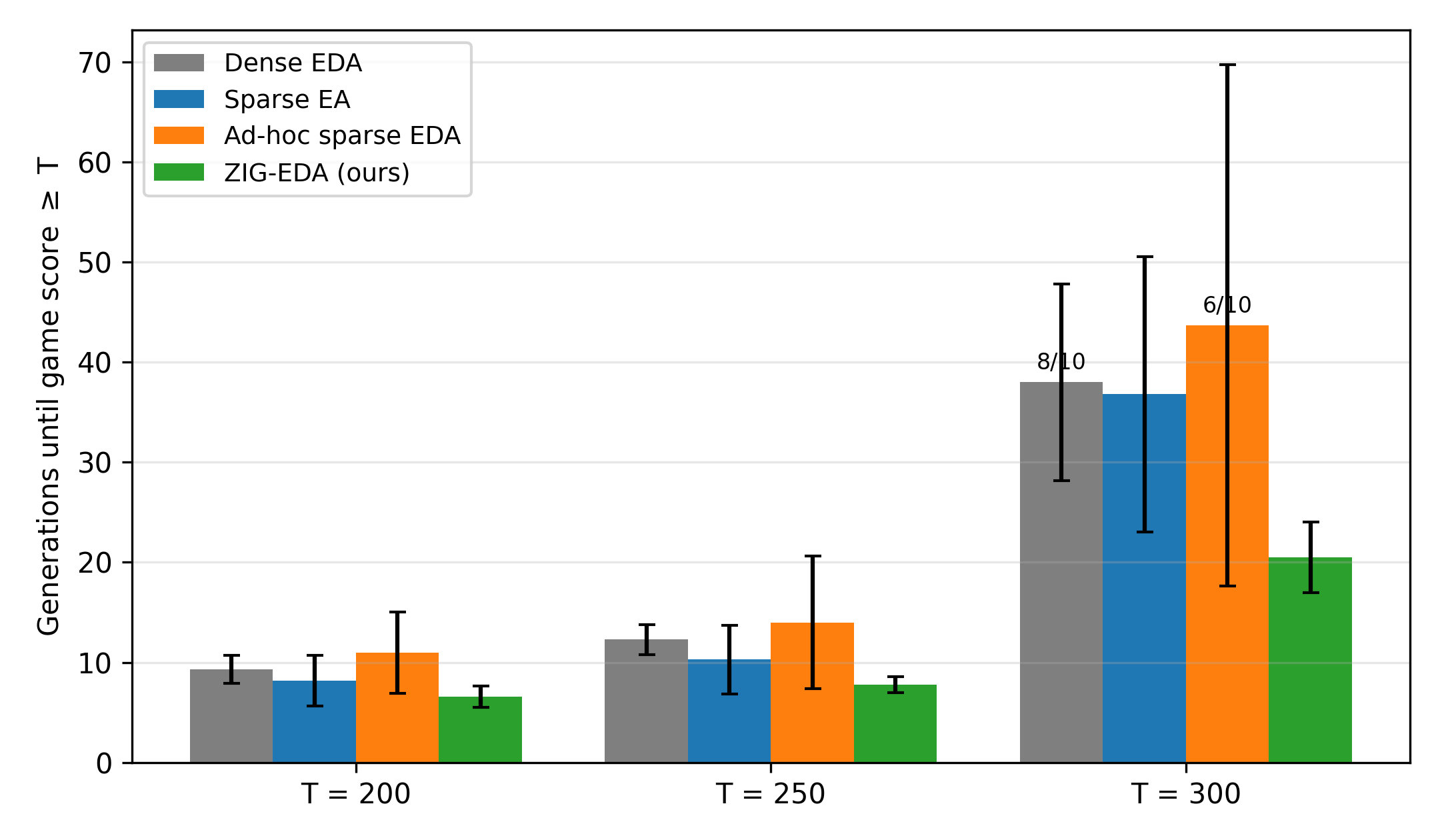}
        \caption{Number of generations until the environment return first reaches the threshold $T$, for $T \in \{200, 250, 300\}$ (mean over the runs that reached the threshold; error bars show one standard deviation). Annotations mark cases where not all $10$ runs reached the threshold.}
        \label{fig:arrival_times}
    \end{minipage}\hfill
    \begin{minipage}[t]{0.48\linewidth}
        \centering
        \includegraphics[width=\linewidth]{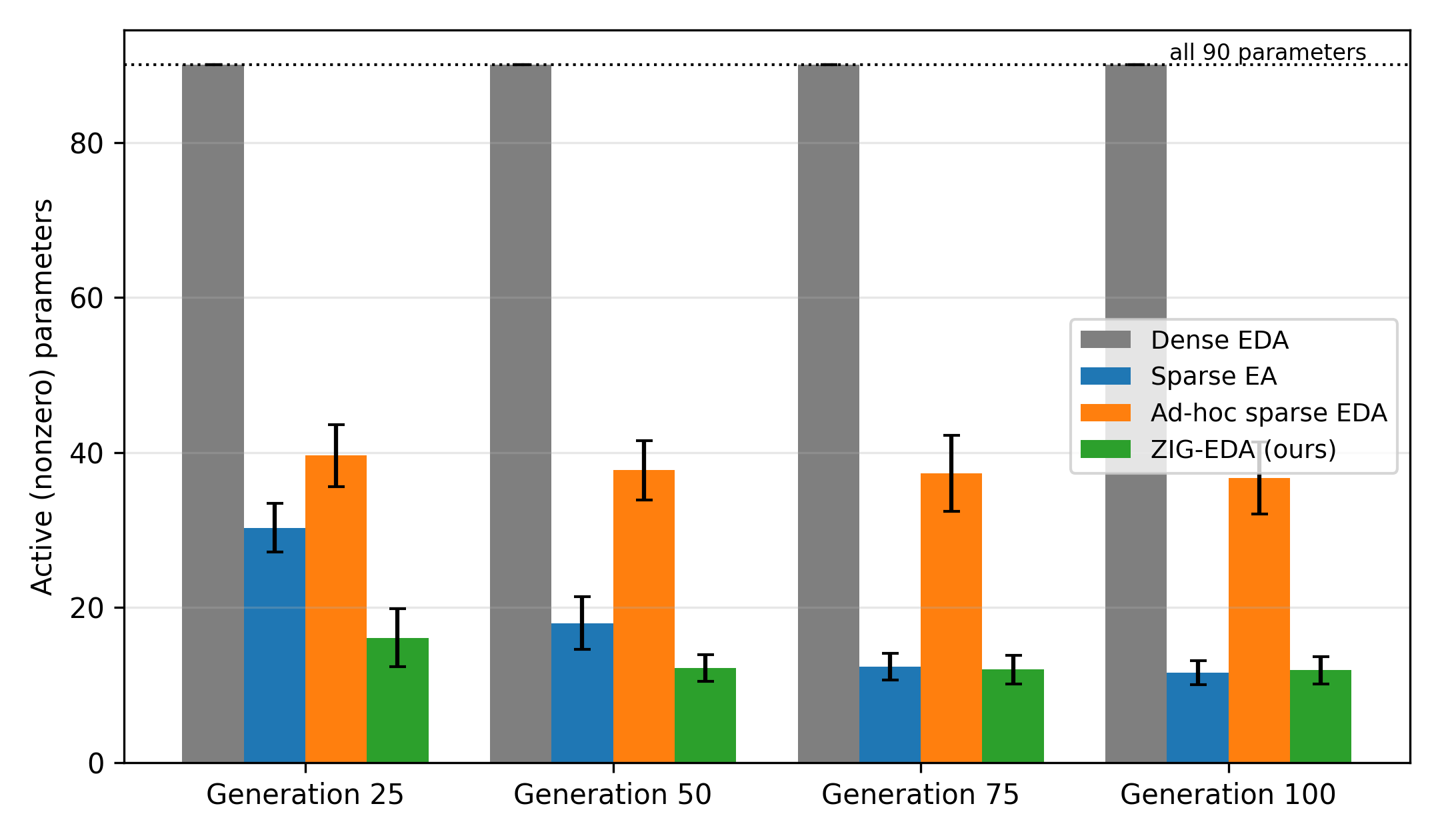}
        \caption{Number of active parameters (out of $90$) at fixed generations (mean over $10$ runs; error bars show one standard deviation).}
        \label{fig:sparsity_bars}
    \end{minipage}
\end{figure}

Three observations stand out. First, the ZIG-EDA converges fastest: it reaches every score threshold earlier than all baselines (Figure~\ref{fig:arrival_times}), needing on average $20.5$ generations to exceed a return of $300$, compared to $36.8$ for the sparse EA and $38$ for the dense EDA---and it is the only EDA variant for which all $10$ runs reached that threshold. Second, it achieves the highest and most consistent final scores ($308.0 \pm 3.5$, Figure~\ref{fig:final_scores}). Notably, it outperforms the dense EDA on this task, which we did not anticipate: for this particular optimization problem, sparse solutions appear to generalize better across random environment initializations, which underscores the regularization capabilities of sparsity. Third, the ZIG-EDA discovers sparse structure quickly: by generation $25$ its best controllers use on average $16$ of $90$ parameters, a level the hand-crafted sparse EA only reaches around generation $50$ (Figure~\ref{fig:sparsity_bars}); both converge to about $12$ active parameters.

The ad-hoc sparse EDA performs worst on all three measures, despite being able to represent exact zeros and joint statistics over its doubled genome. This supports the analysis in Section~\ref{sec:baselines}: estimating the correlation structure over the dense $(w, t)$ parametrization mixes the values of inactive parameters and noisy thresholds into the statistics, which obscures the dependencies that matter for the search. Modeling sparsity explicitly in the sampling law, as the ZIG-EDA does, avoids this problem at the root.

\FloatBarrier

\subsubsection{Recovered controller}
\label{sec:recovered-controller}

Beyond the aggregate comparison, we examine a single controller recovered by the ZIG-EDA in the configuration of Algorithm~\ref{alg:eda-lander}. Figure~\ref{fig:eda_progress} shows the progression of this run. The optimized policy is a sparse quadratic controller with
only $13$ active coefficients out of $90$. Despite its compactness, this controller
is sufficient to achieve stable landings (Figure~\ref{fig:eda_trajectory}). The explicit formulas for the two action
dimensions are given below, corresponding to the main and side engine throttles.
The raw outputs are subsequently clamped to the valid range $[-1,1]$. Notably, the learned policy does not make use of the \texttt{leg\_1}
contact indicator at all.

\[
\begin{aligned}
a_{\text{main}}
  &= -2.5988\,y \;-\; 1.7162\,\dot{y}
     \;-\; 0.0226\,x\,\text{leg}_2
     \;-\; 2.8239\,y^2
     \;+\; 4.0005\,\dot{y}^2
     \;-\; 0.4014\,\text{leg}_2^2, \\[0.75em]
a_{\text{side}}
  &= -1.6502\,\dot{x}
     \;+\; 3.6608\,\theta
     \;-\; 2.1000\,x y
     \;+\; 2.7215\,y \dot{\theta}
     \;+\; 3.0316\,\dot{x}\dot{y}
     \;-\; 3.3986\,\dot{y}\dot{\theta}
     \;+\; 0.0506\,\theta\,\text{leg}_2.
\end{aligned}
\]

\begin{figure}[htbp]
    \centering
    \fbox{\includegraphics[width=0.6\linewidth]{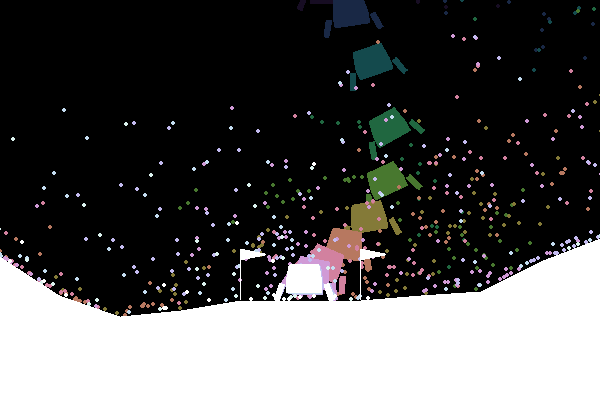}}
    \caption{Temporal overlay of a successful landing episode under the optimized quadratic controller.
    Colors indicate temporal progression.}
\label{fig:eda_trajectory}
\end{figure}

\begin{figure}[htbp]
    \centering
    \includegraphics[width=0.8\linewidth]{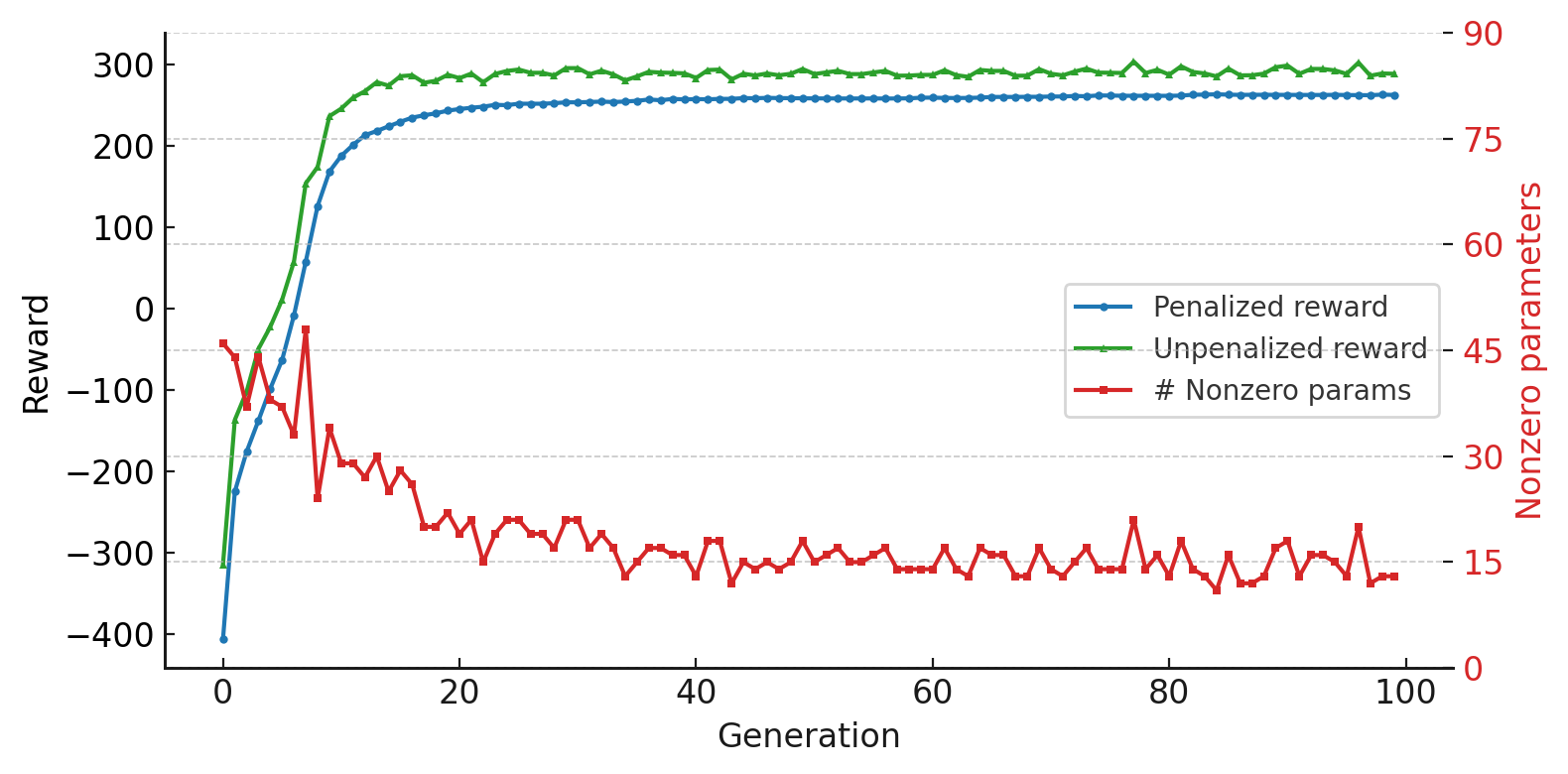}
\caption{Progress of the ZIG-EDA optimization over 100 generations in the configuration of Algorithm~\ref{alg:eda-lander}. The blue curve shows the penalized fitness, i.e.\ the environment return reduced by a penalty of $2$ per active parameter. The green curve shows the corresponding unpenalized return, reflecting the raw environment performance of the controller. The red curve indicates the number of nonzero parameters (out of 90). Gridlines correspond to parameter counts. Solved performance (about $200$ unpenalized return) is reached after only $9$ generations.}
    \label{fig:eda_progress}
\end{figure}

\FloatBarrier

\section{Conclusion}
\label{sec:conclusion}

Sparse black-box optimization previously required baking assumptions into the optimizer: hand-crafted sparsity-aware mutation operators, bi-level schemes alternating between support set and active values, zeroing thresholds, or switching schedules. The central takeaway of this work is that none of these are necessary. By equipping an EDA with a zero-inflated Gaussian sampling law, the defining property of EDAs---that the statistics of the best solutions found so far drive the search, with no operator design---extends to sparse parameter spaces. Explicitly separating value and mask latents lets the model capture dependencies among continuous
parameters, sparsity indicators, and their interactions, so that
sparsity patterns and active parameter values are optimized simultaneously
rather than in alternating stages. The one assumption we do retain is distributional---active values are modeled as correlated Gaussians---but it concerns the search distribution, not the structure of the problem, and is the same assumption that dense Gaussian EDAs already make.

We developed inversion-based estimators for recovering model parameters from data
and showed empirically that the approach accurately recovers latent correlation
structures in high-dimensional settings. Applied to the Lunar Lander benchmark,
the proposed ZIG-EDA converged faster and achieved higher final returns than a dense Gaussian EDA, a hand-crafted sparse evolutionary algorithm, and an ad-hoc sparse EDA, while identifying compact quadratic controllers that achieve stable
landings with only a small fraction of active parameters.

\paragraph{Limitations.}
We do not provide a proof of convergence or a sample-efficiency analysis for the resulting EDA. Convergence proofs for EDAs exist but typically address simplified settings---single-variable problems, additively decomposable functions, or infinite-population limits \citep{KrejcaWitt2020theory}---and do not extend to the high-dimensional, non-separable, and noisy settings that motivate our work. We therefore follow established practice in this area and rely on empirical evaluation. Furthermore, as shown in Section~\ref{sec:exp-recovery}, correlation recovery degrades for rarely active dimensions; this is a sample-efficiency effect rather than a structural limitation, but it implies that very high-dimensional, strongly correlated regimes may require additional regularization or structure on $\Sigma$.

\paragraph{Future work.}
Natural directions include extending the estimators to more general continuous components, such as copula-based marginals; investigating scalability to larger problem classes; and embedding the model within more advanced EDA schemes. In particular, combining the zero-inflated Gaussian sampling law with CMA-ES--style covariance adaptation---iteratively refining the latent correlation matrix with small, regularized updates instead of refitting it from scratch each generation---could reduce estimation noise in high dimensions while retaining the expressiveness of the full latent correlation model. Restricting $\Sigma$ to structured families (low-rank, block, or sparse precision structure) would recover several classical structured EDAs as special cases of the framework presented here.

\clearpage
\bibliographystyle{plainnat}
\bibliography{refs}

\clearpage
\appendix

\section{Attenuation of Empirical Correlations}
\label{appendix:attenuation}

This appendix provides exact integral formulations for the expected empirical correlations
\[
\mathbb{E}\!\left[P(i,j)\right]
= f_{ij}\!\left(\Sigma(i,j),\,p_i,\,p_j,\ldots\right)
\]
across all three block types (\(MM\), \(VM\), \(VV\)), and establishes their strict monotonicity with respect to the corresponding latent correlations \(\Sigma(i,j)\).
All random variables are standard normal with mean \(0\) and variance \(1\).

\subsection{Notation and conditional identities}

\paragraph{Variables.}
For each index \(i\):
\[
V_i,\ M_i \sim \mathcal N(0,1), \quad
X_i = V_i \mathbf{1}\{M_i \ge \tau_i\},\quad
I_i = \mathbf{1}\{M_i \ge \tau_i\}.
\]

\paragraph{Thresholds and probabilities.}
\[
p_i = \Pr(M_i \ge \tau_i) = \bar{\Phi}(\tau_i),\qquad
\bar{\Phi}(\tau) = 1 - \Phi(\tau).
\]

\paragraph{Correlations.}
For distinct indices \(i,j\):
\[
\Sigma_{VV}(i,j)=\rho_V,\quad
\Sigma_{MM}(i,j)=\rho_M,\quad
\Sigma_{VM}(i,j)=\rho_{ij},\quad
\Sigma_{MV}(j,i)=\rho_{ji},
\]
with \(\Sigma_{VM}(i,i)=0\).

\paragraph{Bivariate Gaussian pdf.}
\[
\varphi_2(m_i,m_j;\rho_M)
=\frac{1}{2\pi\sqrt{1-\rho_M^2}}
\exp\!\left(
-\frac{m_i^2+m_j^2-2\rho_M m_i m_j}{2(1-\rho_M^2)}
\right).
\]

\paragraph{Joint survival function.}
\[
\bar{\Phi}_2(\tau_i,\tau_j;\rho_M)
=\int_{\tau_i}^{\infty}\!\int_{\tau_j}^{\infty}
\varphi_2(m_i,m_j;\rho_M)\,dm_j\,dm_i.
\]

\paragraph{Conditional means.}
For any trivariate normal \((V_i,M_i,M_j)\) with \(\operatorname{corr}(V_i,M_i)=0\),
\[
\mathbb{E}[V_i \mid M_i=m_i,\,M_j=m_j]
=\frac{\rho_{ji}}{1-\rho_M^2}\big(m_j - \rho_M m_i\big).
\]

For the bivariate pair \((M_i,M_j)\):
\[
\mathbb{E}[M_j - \rho_M M_i \mid M_i=m_i]
=\sqrt{1-\rho_M^2}\,\varphi\!\left(
\frac{\tau_j - \rho_M m_i}{\sqrt{1-\rho_M^2}}
\right)\Big/\bar{\Phi}\!\left(
\frac{\tau_j - \rho_M m_i}{\sqrt{1-\rho_M^2}}
\right)\!>\!0,
\]
implying the conditional factor in all cross-integrals is positive.

\subsection{Mask--mask correlations}

\[
I_i = \mathbf{1}\{M_i \ge \tau_i\},
\quad
I_j = \mathbf{1}\{M_j \ge \tau_j\}.
\]

\paragraph{Expectation and correlation.}
\[
\mathbb{E}[I_i] = p_i,
\quad
\mathbb{E}[I_j] = p_j,
\quad
\mathbb{E}[I_i I_j] = \bar{\Phi}_2(\tau_i,\tau_j;\rho_M),
\]
hence
\[
\boxed{
\operatorname{corr}(I_i, I_j)
=
\frac{\bar{\Phi}_2(\tau_i,\tau_j;\rho_M) - p_i p_j}
{\sqrt{p_i(1-p_i)\,p_j(1-p_j)}}
}.
\]

\paragraph{Monotonicity.}
\begin{theorem}
For any thresholds \(\tau_i,\tau_j\),
\(\operatorname{corr}(I_i, I_j)\) is strictly increasing in \(\rho_M \in (-1,1)\).
\end{theorem}

\begin{proof}
The numerator \(\bar{\Phi}_2(\tau_i,\tau_j;\rho_M) - p_i p_j\)
is strictly increasing in \(\rho_M\) (Gaussian quadrant probability increases with correlation),
while the denominator is constant. Thus the derivative is positive.
\end{proof}

\subsection{Value--mask correlations}

\[
X_i = V_i \mathbf{1}\{M_i \ge \tau_i\}, \quad
I_j = \mathbf{1}\{M_j \ge \tau_j\}.
\]

\paragraph{Moment relations.}
\[
\mathbb{E}[X_i]=0,\quad
\operatorname{Var}(X_i)=p_i,\quad
\operatorname{Var}(I_j)=p_j(1-p_j).
\]

\paragraph{Covariance.}
Using \(\mathbb{E}[V_i \mid M_i,M_j]\),
\[
\mathbb{E}[X_i I_j]
=
\int_{\tau_i}^{\infty}\!\!\int_{\tau_j}^{\infty}
\frac{\rho_{ji}}{1-\rho_M^2}
\big(m_j - \rho_M m_i\big)
\varphi_2(m_i,m_j;\rho_M)\,dm_j\,dm_i.
\]
Thus
\[
\boxed{
\operatorname{corr}(X_i, I_j)
=
\frac{1}{\sqrt{p_i\,p_j(1-p_j)}}
\cdot
\frac{\rho_{ji}}{1-\rho_M^2}
\int_{\tau_i}^{\infty}\!\!\int_{\tau_j}^{\infty}
(m_j - \rho_M m_i)
\varphi_2(m_i,m_j;\rho_M)\,dm_j\,dm_i
}.
\]

\paragraph{Monotonicity.}
\begin{theorem}
For fixed \(\rho_M,\tau_i,\tau_j\),
\(\operatorname{corr}(X_i,I_j)\) is strictly increasing in \(\rho_{ji}\).
\end{theorem}

\begin{proof}
The integral term
\[
J(\tau_i,\tau_j;\rho_M)
=\int_{\tau_i}^{\infty}\!\!\int_{\tau_j}^{\infty}
(m_j-\rho_M m_i)\varphi_2(m_i,m_j;\rho_M)\,dm_j\,dm_i
\]
is strictly positive (conditional expectation argument above),
and the denominator is constant in \(\rho_{ji}\).
Hence \(\operatorname{corr}(X_i,I_j)\) is an affine, strictly increasing function of \(\rho_{ji}\).
\end{proof}

\subsection{Value--value correlations}
\label{appendix:vv}

\[
X_i = V_i \mathbf{1}\{M_i \ge \tau_i\},
\quad
X_j = V_j \mathbf{1}\{M_j \ge \tau_j\}.
\]

\paragraph{Parameters.}
\[
\rho_V = \operatorname{corr}(V_i,V_j),\quad
\rho_M = \operatorname{corr}(M_i,M_j),\quad
\rho_{ij} = \operatorname{corr}(M_i,V_j),\quad
\rho_{ji} = \operatorname{corr}(M_j,V_i).
\]

\paragraph{Derivation sketch.}
The expected correlation follows from the Gaussian law of total covariance, conditioning on the pair of mask-latents:
\begin{enumerate}[itemsep=0.2em]
    \item Start from \(\mathbb{E}[X_i X_j] = \mathbb{E}\big[V_i V_j \,\mathbf{1}\{M_i \ge \tau_i\}\mathbf{1}\{M_j \ge \tau_j\}\big]\). Conditioning on \((M_i, M_j)\) makes the indicators deterministic.
    \item Apply the law of total covariance:
    \(\mathbb{E}[V_i V_j \mid M_i, M_j] = \operatorname{Cov}(V_i, V_j \mid M_i, M_j) + \mathbb{E}[V_i \mid M_i, M_j]\,\mathbb{E}[V_j \mid M_i, M_j]\).
    \item The conditional covariance of \((V_i, V_j)\) and the conditional means come from the standard Gaussian-conditioning formulas and are explicit affine functions of \((m_i, m_j)\). This is where the constant term and the quadratic term in \((m_i, m_j)\) below originate.
    \item Integrating these conditional expressions over the truncated bivariate Gaussian region \([\tau_i,\infty) \times [\tau_j,\infty)\) yields exactly the two-term structure displayed next: a term proportional to \(\rho_V\) (plus a correction from the conditional covariance) multiplied by the survival probability, and a second term proportional to \(\rho_{ij}\rho_{ji}\) coming from the product of conditional means.
\end{enumerate}
Conceptually, masking only attenuates the latent dependence; it cannot invert its ordering. This is why, with all other relevant latent correlations and thresholds held fixed, the value--value correlation is an affine, strictly increasing function of \(\rho_V\), as proven below.

\paragraph{Expected correlation.}
\[
\boxed{
\begin{aligned}
\operatorname{corr}(X_i,X_j)
&=\frac{1}{\sqrt{p_i p_j}}
\Bigg[
\left(\rho_V+\frac{\rho_M\rho_{ij}\rho_{ji}}{1-\rho_M^2}\right)
\bar{\Phi}_2(\tau_i,\tau_j;\rho_M)
\\[0.4em]
&\hspace{2em}
+\frac{\rho_{ij}\rho_{ji}}{(1-\rho_M^2)^2}
\!\!\int_{\tau_i}^{\infty}\!\!\int_{\tau_j}^{\infty}
\!\!\!\big(-\rho_M(m_i^2+m_j^2)+(1+\rho_M^2)m_i m_j\big)\,
\varphi_2(m_i,m_j;\rho_M)\,dm_j\,dm_i
\Bigg].
\end{aligned}
}
\]

\paragraph{Monotonicity in \(\rho_V\).}
\begin{theorem}
Fix thresholds \(\tau_i,\tau_j\) and correlations \(\rho_M,\rho_{ij},\rho_{ji}\) with \(0<p_i,p_j<1\).
Then \(\operatorname{corr}(X_i,X_j)\) is a strictly increasing affine function of \(\rho_V\) on any interval where the latent covariance \(\Sigma\) is positive semidefinite.
\end{theorem}

\begin{proof}
From the displayed formula for \(\operatorname{corr}(X_i,X_j)\), isolate the \(\rho_V\)-dependence:
\[
\operatorname{corr}(X_i,X_j)
=
\frac{\bar{\Phi}_2(\tau_i,\tau_j;\rho_M)}{\sqrt{p_i p_j}}\;\rho_V
\;+\; C,
\]
where the constant $C$ does not depend on \(\rho_V\).

Hence
\[
\frac{\partial}{\partial \rho_V}\,\operatorname{corr}(X_i,X_j)
=
\frac{\bar{\Phi}_2(\tau_i,\tau_j;\rho_M)}{\sqrt{p_i p_j}}.
\]
Because \(-\infty<\tau_i,\tau_j<\infty\) and the bivariate normal density is strictly positive,
\(\bar{\Phi}_2(\tau_i,\tau_j;\rho_M)=\Pr(M_i\ge \tau_i,M_j\ge \tau_j)>0\).
Also \(p_i,p_j\in(0,1)\), so \(\sqrt{p_i p_j}>0\).
Therefore the derivative is strictly positive and constant in \(\rho_V\),
implying strict monotonicity on any admissible (PSD) range of \(\rho_V\).
\end{proof}

\section{Identifiability under ICIN}
\label{appendix:identifiability}

To identify all parameters of the latent Gaussian zero-inflated model from observed data, we require a structural assumption relating the value-latents $V_i$ and mask-latents $M_i$. The \emph{itemwise conditional independence} (ICIN) assumption
\[
\operatorname{corr}(V_i, M_i) = 0 \quad \text{for all } i \in \{1,\dots,d\}
\]
excludes direct dependence between the latent value and mask of the same observed dimension. This section clarifies why ICIN is both necessary and sufficient for identifiability.

\subsection{ICIN is necessary}

Without ICIN, correlations between $V_i$ and $M_i$ introduce bias in the observed mean and variance of the Gaussian component, making multiple latent parameter triples $(\mu_i,\sigma_i,\gamma_i)$ consistent with the same observed distribution.

\paragraph{Proof by contradiction.}
Consider a single dimension, with activation probability $p\in(0,1)$ and threshold $\tau=\Phi^{-1}(1-p)$. Let
\[
(V, M) \sim \mathcal N\!\left(0,
\begin{psmallmatrix}
1 & \gamma \\
\gamma & 1
\end{psmallmatrix}\right),
\qquad
X =
\begin{cases}
0, & M < \tau,\\[0.2em]
\mu + \sigma V, & M \ge \tau.
\end{cases}
\]
Let $\lambda=\phi(\tau)/(1-\Phi(\tau))$ and $b=1+\tau\lambda-\lambda^2$.
The conditional mean and variance of the observed nonzero values are
\[
\mathbb{E}[X \mid X\neq 0] = \mu + \sigma \gamma \lambda,
\qquad
\operatorname{Var}(X \mid X\neq 0) = \sigma^2 \big(1 - \gamma^2(1-b)\big).
\]
Now fix observed statistics $(m_\star,s_\star)$ of the nonzero values and activation probability $p$.
For any $\gamma\in(-1,1)$, define
\[
\sigma(\gamma)=\sqrt{\frac{s_\star}{1-\gamma^2(1-b)}},
\qquad
\mu(\gamma)=m_\star-\sigma(\gamma)\gamma\lambda.
\]
Each choice of $\gamma$ yields parameters $(\mu(\gamma), \sigma(\gamma), \gamma)$ producing the same $(m_\star, s_\star, p)$.
Therefore, $(\mu, \sigma, \gamma)$ is not uniquely determined---contradicting identifiability.
Hence, ICIN (i.e.\ $\gamma=0$) is necessary.

\subsection{ICIN is sufficient}

Under ICIN, the value- and mask-latents of each dimension are uncorrelated.
The marginal parameters $(\mu_i, \sigma_i, \tau_i)$ are then directly identified from the observed data: $\tau_i$ from the empirical activation probability, and $(\mu_i, \sigma_i)$ from the mean and variance of the nonzero entries.

With all marginal parameters known, the latent correlation matrix $\Sigma$ can be uniquely recovered from the empirical correlation matrix $P$ by applying the monotone inverse mappings derived in Section~\ref{sec:attenuation} and Appendix~\ref{appendix:attenuation}.
These mappings establish a one-to-one correspondence between each empirical block $(P_{VV}, P_{VM}, P_{MM})$ and its latent counterpart $(\Sigma_{VV}, \Sigma_{VM}, \Sigma_{MM})$, confirming sufficiency of ICIN for full model identifiability.

\section{Training Details}
\label{appendix:training}

This appendix provides additional information on the generation of synthetic training data, the network architectures used for learning the inverse mappings, and the training procedure including loss function, optimizer, and hyperparameters.

\subsection{Data generation}

Synthetic training data were generated to cover the observable manifolds of all three correlation block types (\(MM\), \(VM\), \(VV\)).
Each data point corresponds to a pair of empirical correlations and its latent counterpart, obtained through simulation under known latent parameters.

\paragraph{Latent correlation sampling.}
Each training instance corresponds to a two-dimensional zero-inflated Gaussian model as defined in the main text, with a total of four latent variables \((V_i, V_j, M_i, M_j)\).
Their joint correlation structure is represented by a \(4 \times 4\) latent correlation matrix
\[
\Sigma =
\begin{pmatrix}
1 & a & 0 & b \\
a & 1 & c & 0 \\
0 & c & 1 & d \\
b & 0 & d & 1
\end{pmatrix},
\]
where the quadrants correspond to the latent blocks
\[
\Sigma_{VV} = \begin{psmallmatrix}1 & a \\ a & 1\end{psmallmatrix}, \quad
\Sigma_{VM} = \begin{psmallmatrix}0 & b \\ c & 0\end{psmallmatrix}, \quad
\Sigma_{MM} = \begin{psmallmatrix}1 & d \\ d & 1\end{psmallmatrix}.
\]
A dense grid over \((a,b,c,d) \in [-1,1]^4\) was sampled using linearly spaced values.
Each combination was tested for positive semidefiniteness via an eigenvalue check, and only those with nonnegative eigenvalues were retained as valid latent correlation matrices.

\paragraph{Empirical observation sampling.}
For each valid \(\Sigma\), multiple activation probability pairs \((p_i, p_j)\) were drawn from a uniform grid in \([0,1]\), ensuring broad coverage across activation levels.
Given \((p_i,p_j)\), thresholds \(\tau_i=\Phi^{-1}(1-p_i)\) and \(\tau_j=\Phi^{-1}(1-p_j)\) were computed, and \(n=10^5\) samples were drawn from the corresponding four-dimensional Gaussian latent model.
Observed variables were obtained by thresholding the mask-latents, setting inactive entries to zero.
Empirical correlations \(P_{\cdot\cdot}\) were then computed from these samples and paired with the corresponding latent correlations \(\Sigma_{\cdot\cdot}\) as supervised targets.

\paragraph{Dataset composition.}
The aggregated dataset consists of independent training examples spanning the feasible correlation space.
Each entry contains the relevant activation probabilities and empirical correlations as input features, and the associated latent correlation value as target.
This ensures dense and uniform coverage of the inverse mapping domains for all three block types.

\subsection{Network architecture}

Each inverse mapping was implemented as a compact multilayer perceptron (MLP) trained to predict a latent correlation value from observable quantities such as activation probabilities and empirical correlations.
Separate networks were trained for each block type (\(MM, VM, VV\)), differing only in input dimensionality.

\paragraph{Model structure.}
The MLP consists of four hidden layers, organized into two functional stages:
\[
x \;\mapsto\;
h_1 = \frac{W_1 x}{\operatorname{softplus}(W_2 x) + 0.1},
\quad
h_2 = \frac{W_3 h_1}{\operatorname{softplus}(W_4 h_1) + 0.1},
\quad
\hat{y} = \tanh(W_{\text{out}} h_2).
\]
This structure introduces elementwise divisions by smooth positive functions, acting as adaptive normalizers and stabilizing the learned mappings near correlation boundaries.
The final output is squashed to \([-1,1]\) via a hyperbolic tangent, matching the valid range of correlation coefficients.

\paragraph{Design rationale.}
The architecture was identified through manual model search as the best-performing variant in terms of validation accuracy.
Mappings of the form \((p_i, p_j, P_{\cdot\cdot}) \mapsto \Sigma_{\cdot\cdot}\) exhibit steep nonlinearities near the edges of the probability domain,
specifically for \(p_i, p_j \in [0, \varepsilon] \cup [1-\varepsilon, 1]\) with small \(\varepsilon\).
In these boundary regions, small changes in the inputs can lead to pronounced variations in the latent correlation.
The use of multiplicative normalization and smooth nonlinearities enables the network to model such sharp transitions while maintaining monotonicity across the full domain.

\paragraph{Dimensionality and precision.}
The input dimensionality varies with the block type and reflects the number of observables in each mapping.
Hidden layers contain 32 units each, providing sufficient expressivity for smooth monotonic mappings while remaining lightweight for efficient evaluation.
All models were trained in \texttt{torch.float32} with high matrix multiplication precision enabled (\texttt{torch.set\_float32\_matmul\_precision('high')}) to ensure stable convergence and consistent estimates across hardware platforms.

\subsection{Training procedure}

Each network was trained in supervised fashion to regress from observable features
(e.g.\ activation probabilities and empirical correlations)
to the corresponding latent correlation coefficient.
All training was performed using synthetic data only, with independent models fitted for each block type.

\paragraph{Loss and optimization.}
Training minimized the mean-squared error (MSE) between the predicted and true latent correlations.
Optimization used the Adam algorithm with a learning rate of \(10^{-3}\),
and mini-batches of size 128.
Gradients were backpropagated through all layers, and parameters were updated once per batch.

\paragraph{Data splitting and schedule.}
For each block type, the full dataset was randomly divided into 80\% training and 20\% validation subsets.
Validation loss was monitored throughout training, and mean absolute deviation was tracked as a secondary stability metric.
Training was conducted for a fixed number of epochs per block type, reflecting convergence behavior observed during experimentation:
\[
\text{MM: }1200~\text{epochs},
\quad \text{VM: }250~\text{epochs},
\quad \text{VV: }250~\text{epochs.}
\]
All runs employed shuffled mini-batches to ensure uniform coverage of the input space.

\paragraph{Hardware and precision.}
Models were trained on GPU when available (\texttt{device="cuda"}),
using standard \texttt{torch.float32} arithmetic with high matrix multiplication precision enabled.
On an RTX~3090, complete training of all three block-type networks can be performed in under two hours.

\section{Sparse EA Baseline Details}
\label{appendix:sparse-ea}

This appendix specifies the hand-crafted sparse evolutionary algorithm used as a baseline in Section~\ref{sec:baselines}. The algorithm shares its outer loop with all other methods: the same initial population (sampled from a zero-inflated Gaussian with $\mu_0{=}0$, $\sigma_0{=}2$, $p_0{=}0.5$, $\Sigma_0{=}I$), the same archive-based top-$k$ selection over all previously evaluated individuals, the same population size ($N{=}500$), elite size ($K{=}100$), number of generations ($T{=}100$), and the same penalized fitness. It differs only in the variation step: instead of refitting a distribution to the elites and sampling from it, each of the $N{-}K$ offspring is generated by drawing two parents uniformly at random (with replacement) from the $K$ elites, applying crossover, and then applying mutation.

\paragraph{Crossover.}
With probability $1/2$, uniform crossover is performed: each coordinate of the child is copied from the first or the second parent with equal probability. Otherwise, the child is a plain copy of the first parent. Because coordinates are copied verbatim, exact zeros are inherited, so crossover recombines sparsity patterns as well as active values.

\paragraph{Mutation.}
Each offspring undergoes exactly one mutation event. The event type is drawn from three categories with fixed weights $(w_{\text{off}}, w_{\text{on}}, w_{\text{perturb}}) = (1, 1, 10)$, normalized to probabilities $(1/12,\, 1/12,\, 5/6)$:
\begin{itemize}[itemsep=0.2em]
    \item \textbf{Deactivation} (probability $1/12$): a uniformly chosen \emph{active} coordinate is set to exactly zero.
    \item \textbf{Activation} (probability $1/12$): a uniformly chosen \emph{inactive} coordinate is assigned a fresh value drawn from $\mathcal{N}(0, \sigma_{\text{init}}^2)$ with $\sigma_{\text{init}} = 2$, matching the initialization scale.
    \item \textbf{Perturbation} (probability $5/6$): a uniformly chosen \emph{active} coordinate receives additive Gaussian noise drawn from $\mathcal{N}(0, \sigma_{\text{noise}}^2)$ with $\sigma_{\text{noise}} = 0.1$.
\end{itemize}
If the drawn event is not applicable---deactivation or perturbation when no coordinate is active, or activation when none is inactive---a uniformly chosen coordinate (active or not) is instead perturbed with noise from $\mathcal{N}(0, \sigma_{\text{noise}}^2)$.

Note that this baseline embodies exactly the kind of operator design the ZIG-EDA avoids: the on/off moves, their relative frequencies, the re-activation scale $\sigma_{\text{init}}$, and the perturbation scale $\sigma_{\text{noise}}$ are all hand-chosen and couple the algorithm to the problem instance.

\section{Full Baseline Comparison Tables}
\label{appendix:tables}

This appendix tabulates the complete numerical results of the Lunar Lander comparison study described in Section~\ref{sec:baselines}. All values are means $\pm$ standard deviations over $10$ independent runs.

\begin{table}[h]
\caption{Final unpenalized scores in the Lunar Lander task after $100$ generations.}
\label{tab:final_scores}
\begin{center}
\begin{tabular}{lc}
\toprule
\textbf{Method} & \textbf{Mean $\pm$ Std} \\
\midrule
Dense EDA & $301.63 \pm 7.71$ \\
Sparse EA & $303.59 \pm 4.73$ \\
Ad-hoc sparse EDA & $293.69 \pm 13.18$ \\
ZIG-EDA (ours) & $\mathbf{308.03 \pm 3.51}$ \\
\bottomrule
\end{tabular}
\end{center}
\end{table}

\begin{table}[h]
\caption{Arrival times: number of generations needed until the environment return first reaches the threshold $T$. Means are taken over the runs that reached the threshold; cases with fewer than $10$ such runs are annotated.}
\label{tab:arrival_times}
\begin{center}
\begin{tabular}{lccc}
\toprule
\textbf{Method} & $T=200$ & $T=250$ & $T=300$ \\
\midrule
Dense EDA & $9.30 \pm 1.42$ & $12.30 \pm 1.49$ & $38.00 \pm 9.83$ {\footnotesize(8/10 runs)} \\
Sparse EA & $8.20 \pm 2.53$ & $10.30 \pm 3.43$ & $36.80 \pm 13.74$ \\
Ad-hoc sparse EDA & $11.00 \pm 4.06$ & $14.00 \pm 6.60$ & $43.67 \pm 26.01$ {\footnotesize(6/10 runs)} \\
ZIG-EDA (ours) & $\mathbf{6.60 \pm 1.07}$ & $\mathbf{7.80 \pm 0.79}$ & $\mathbf{20.50 \pm 3.54}$ \\
\bottomrule
\end{tabular}
\end{center}
\end{table}

\begin{table}[h]
\caption{Number of active parameters (out of $90$) of the best individual at fixed generations.}
\label{tab:sparsity}
\begin{center}
\begin{tabular}{lcccc}
\toprule
\textbf{Method} & \textbf{Gen 25} & \textbf{Gen 50} & \textbf{Gen 75} & \textbf{Gen 100} \\
\midrule
Dense EDA & $90.00 \pm 0.00$ & $90.00 \pm 0.00$ & $90.00 \pm 0.00$ & $90.00 \pm 0.00$ \\
Sparse EA & $30.30 \pm 3.13$ & $18.00 \pm 3.40$ & $12.40 \pm 1.71$ & $11.60 \pm 1.58$ \\
Ad-hoc sparse EDA & $39.60 \pm 3.98$ & $37.70 \pm 3.80$ & $37.30 \pm 4.90$ & $36.70 \pm 4.67$ \\
ZIG-EDA (ours) & $\mathbf{16.10 \pm 3.73}$ & $\mathbf{12.20 \pm 1.75}$ & $12.00 \pm 1.83$ & $11.90 \pm 1.79$ \\
\bottomrule
\end{tabular}
\end{center}
\end{table}

\end{document}